\documentclass[twoside]{article}
\usepackage[left=2cm,right=2cm,top=2cm,bottom=2cm]{geometry}
\usepackage{multicol}
\usepackage{amsmath}
\usepackage{amsfonts}
\usepackage{amssymb}
\usepackage{graphicx}
\usepackage{color}
\usepackage[normalem]{ulem}
\usepackage[utf8]{inputenc}
\usepackage{amsfonts}
\usepackage{color}
\usepackage[normalem]{ulem}
\usepackage{nomencl}
\usepackage{subcaption}
\makenomenclature

\newenvironment{proof}{\noindent{\sc Proof.}}{\qed}
\newtheorem{theorem}{Theorem}[section]
\newtheorem{lemma}{Lemma}[section]

\newtheorem{rem}{Remark}[section]
\newtheorem{definition}{Definition}[section]
\newtheorem{prop}{Proposition}[section]

\newcommand{\qed}{$\blacksquare$}

\def\bhag#1{\noindent
\setcounter{equation}{0}
\section{#1}
}

\def\argmin{\mathop{\hbox{{\rm arg min}}}}

\def\RR{{\mathbb R}}

\def\ZZ{{\mathbb Z}}

\def\bs#1{{\boldsymbol{#1}}}
\def\x{\mathbf{x}}

\def\y{\mathbf{y}}
\def\w{\mathbf{w}}

\def\v{\mathbf{v}}

\def\C{{\mathcal C}}

\def\V{{\cal V}}

\def\W{\mathbf{W}}

\def\argmin{\mathop{\hbox{\textrm{arg min}}}}

\def\be{\begin{equation}}
\def\ee{\end{equation}}
\def\bea{\begin{eqnarray}}
\def\eea{\end{eqnarray}}
\def\eref#1{(\ref{#1})}
\def\disp{\displaystyle}
\def\cfn#1{\chi_{{}_{ #1}}}

\def\donchitre#1#2{\vskip 6.5cm\noindent
\parbox[t]{1in}{\special{eps:#1.eps x=6.5cm y=5.5cm}}
\hbox to 7cm{}\parbox[t]{0.0cm}{\special{eps:#2.eps x=6.5cm y=5.5cm}}}

\def\tn{|\!|\!|}
\def\XX{{\mathbb X}}
\def\BB{{\mathbb B}}

\def\bs#1{{\boldsymbol{#1}}}

\def\ls{\lesssim}
\def\gs{\gtrsim}

\def \x{\mathbf{x}}
\def \w{\mathbf{w}}
\def \v{\mathbf{v}}
\def \y{\mathbf{y}}

\def \U{\mathbf{U}}
\def \V{\mathbf{V}}

\def \W{\mathbf{W}}

\newcommand{\Gc}{\mathcal{G}}

\newcommand{\Tc}{\mathcal{T}}

\title{A Manifold Learning Approach for Gesture Recognition from Micro-Doppler Radar Measurements}
\author{
 E. S. Mason\thanks{
 Hawkeye 360, 196 Van Buren St \#450, Herndon, VA 20170. \textsf{email: eric.mason@he360.com}}\ \  
 H.~N.~Mhaskar
 \thanks{
Institute of Mathematical Sciences, Claremont Graduate University, Claremont, CA 91711, U.S.A.. 
The research of HNM was supported in part by ARO grant W911NF2110218 and NSF DMS grant 2012355.
\textsf{email:} hrushikesh.mhaskar@cgu.edu} 
 \ \
  A. Guo \thanks{
    Mathematics/Computer Science Department, Pomona College, Claremont, CA 91711.
    \textsf{email:} simingadam.guo@pomona.edu, agsmguo@gmail.com} 
  }
  
\date{}
\begin{document}
\maketitle

\begin{abstract}
A recent paper (Neural Networks, {\bf 132} (2020), 253-268) introduces a straightforward and simple kernel based approximation for manifold learning that does not require the knowledge of anything about the manifold, except for its dimension.
In this paper, we examine how the pointwise error in approximation using least squares optimization based on similarly localized kernels depends upon the data characteristics and deteriorates as one goes away from the training data.
The theory is presented with an abstract localized kernel, which can utilize any prior knowledge about the data being located on an unknown sub-manifold of a known manifold.

We demonstrate the performance of our approach using a publicly available micro-Doppler data set, and investigate the use of different preprocessing measures, kernels, and manifold dimensions.
Specifically, it is shown that the localized kernel introduced in the above mentioned paper when used with PCA components leads to a near-competitive performance to deep neural networks, and offers significant improvements in training speed and memory requirements.  
To demonstrate the fact that our methods are agnostic to the domain knowledge, we examine the classification problem in a simple  video data set. 
\end{abstract}


\bhag{Introduction}\label{bhag:introduction}

Identification of hand gestures is a subject of increasing research attention in the area of human-computer interaction, see for example, \cite{ahmed2021hand, skaria2020deep} for  recent reviews. 
Applications include convenient device control, infection prevention in clinical settings, safer and quicker accessibility of features in automotive, and, in a more general sense, detection of unmanned aerial vehicles. 
In contrast to the commonly used hand gesture signal acquisition approaches  such as camera, infra-red sensors, and ultrasonic sensors,  radar sensors are observed to yield a higher performance in adverse lighting conditions and complex background. 
Among the common types of waveforms used by inexpensive K-band radar sensors - pulse, continuous waveform (CW),  and frequency modulated continuous waveform (FMCW)- the FMCW radars provides an improved simultaneous estimation of the range and the Doppler signatures.
Many machine learning strategies are used for the classification of hand gestures, ranging from nearest neighbor and support vector machines (SVMs) to deep convolutional neural networks (CNNs) obtaining an accuracy as high as 95\%. 
Except for deep networks, common machine learning approaches typically requires the extraction of hand crafted features based on domain knowledge.
On the other hand, training a CNN is typically highly time consuming, and the mathematical theory of how, when, and why CNNs will produce the right output is still under development.
In this paper, we demonstrate the use of manifold learning to classify hand gesture signals from FMCW radar data, agnostic of the fact that we are dealing with radar signals (i.e., independently of domain knowledge).
Our general purpose method yields results that compare favorably with other methods in the literature that have been used for micro-Doppler gesture recognition, including CNNs.

The fundamental problem of machine learning is the following. 
We have (training) data  of the form $\{(\x_j,f(\x_j)+\epsilon_j)\}$, where, for some integer $d\ge 1$, $\x_j\in\RR^d$ are realizations of a random variable,  $f$ is an unknown real valued function, and the $\epsilon_j$'s are realizations of a mean zero random variable. 
The distributions of both the random variables are not known.
The problem is to approximate $f$ with a suitable model, especially for the points which are not in the training data.
In the context of classification problems, $f$ is the labeling function defined as follows. If there are $K$ classes, we define
$$
\cfn{i}(\x)=\begin{cases}
	1, &\mbox{if $\x$ belongs to class $i$,}\\
	0, &\mbox{otherwise,}
\end{cases}
$$
and
$$
f(\x)=\argmin_{1\le i\le K}\cfn{i}(\x).
$$
A major theoretical problem in machine learning is the so called curse of dimensionality - the complexity of the approximating model increases exponentially with the dimension of the input. 
Another theoretical problem is that if the class boundaries are not smooth or not well separated then the labeling function is not smooth, resulting again in an increased model complexity.

Manifold learning tries to ameliorate this problem by assuming that the data lies on some unknown, low dimensional manifold, on which the classes are well separated.
With this assumption, there are well developed  techniques to learn various quantities related to the manifold, for example, the eigen-decomposition of the Laplace-Beltrami operator. 
The special issue \cite{achaspissue} gives a good introduction to the  topic of diffusion geometry in this regard.
In \cite{chui_deep, schmidt2019deep}, the authors describe the construction of an atlas on the  manifold which is then used for construction of deep networks.
Since these objects must be learned from the data, these techniques are applicable in the supervised setting, in which we have all the data points $\x_j$ available in advance.

Recently, we have discovered \cite{mhaskar2019deep} a more direct method to approximate functions on unknown manifolds without trying to learn anything about the manifold itself, and giving an approximation with theoretically guaranteed errors on the entire manifold; not just the points in the original data set.
Our approximation has the form
\be\label{eq:nnapprox}
\x\mapsto \sum_j \left(f(\x_j)+\epsilon_j\right)\widetilde{\Phi}_{N,q}(|\x-\x_j|_{2,Q})
\ee
where $\widetilde{\Phi}_{N,q}$ is a specially constructed kernel, $Q$ is the dimension of the ambient space, $q$ is the dimension of the manifold, and $|\circ|_{2,Q}$ is the Euclidean metric on $\RR^Q$ (see Section~\ref{bhag:kernel} for details). 
In the case when $N=1$, our kernel reduces to the oft-used Gaussian kernel.
In this paper, we prove how a more classical approach of empirical risk minimization or generalization error measured in the least squares sense using the kernel introduced in \cite{mhaskar2019deep} leads to errors that show  a gradual deterioration as one moves away from the training data, thereby yielding a heuristic to determine the parameter $N$.

In many applications, we can assume some further knowledge of the data distribution; in particular, that the data is sampled from an unknown sub-manifold of a known manifold. 
It is then natural to conjecture that one might be able to replace the Euclidean distance in \eqref{eq:nnapprox} by the geodesic distance on the known manifold. 
While the theory analoguous to the one in \cite{mhaskar2019deep} seems difficult at this point, we
formulate our main theorems in a greater abstraction to allow the study of pointwise errors using least squares approximation based on the resulting kernels.
We demonstrate experimentally the effect of such a construction by considering each element of the FMCW radar data as a time series, which in turn may be identified with a point on a Grassmann manifold using the ideas in \cite{chellappa}. 

Although our main focus is on the detection of hand gestures based on FMCW radar data, we include an additional example about classification of simple video data in order to demonstrate the fact that our method is agnostic to the domain knowledge feature extraction.

To summarize, the main contributions of this paper are as follows:
\begin{itemize}
	\item We examine pointwise errors in function approximation using empirical risk minimization and generalization error measured with least squared loss with a generalized version of the kernel introduced in \cite{mhaskar2019deep}.
	\item We demonstrate the performance of our theory  in the case of hand gesture recognition based on FMCW radar data. Our results compare favorably with other known techniques, when used with the right feature vectors. 
	Additionally, this provides new benchmark results on the recently published Dop-Net data set \cite{Ritchie2020}.
	\item There are many papers dealing with micro-Doppler gesture recognition using hand-crafted features. As far as we are aware, this is the first work in which we explore the use of the singular vectors and values of the micro-Doppler spectrogram as features. 
	\item Our method is agnostic to the domain knowledge, a fact which we demonstrate using classification in a simple video data.
\end{itemize}

In Section~\ref{bhag:relatedwork}, we summarize the connection between our work and other related works, although not exhaustively.
The technical background that motivates our work in this paper is explained in Section~\ref{bhag:techintro}.
In Section~\ref{bhag:main}, we state our main theoretical results.
The experimental results are presented in Section~\ref{bhag:experiments}. 
The proofs of the results in Section~\ref{bhag:main} are given in Section~\ref{bhag:proofs}.

\bhag{Related Works}\label{bhag:relatedwork}

Micro-Doppler radar has found success in various applications related to classifying humans, animals and objects, such as activity monitoring, autonomous driving, and unmanned aerial vehicle detection \cite{Ritchie2019,Zhao2021,zhou2020,tahmoush2015,molchanov2014}.
The approach taken by practitioners is to use a time-frequency transform of the measured signal, which is then considered the raw data for subsequent processing. The Short-Time Fourier Transform (STFT) is the time-frequency transform commonly selected. Using the spectrograms, features are extracted or the spectrogram is used directly with a neural network.

Hand crafted features are usually inspired by physics and taken to be measurements of frequency and/or time, such as the bandwidth of the micro-Doppler signature, its time duration, energy content, and various spectrogram measurements \cite{bjorklund2015,tahmoush2015,Manfredi2021,Karabacak2015,Ritchie2019}.  
There have also been efforts to define new features through mathematical transforms, such as the Discrete Cosine Transform (DCT), Mel-Frequencies, $S$-transform, etc. \cite{Li2018,Li2017,Erol2018,Ma2019}.
The success of these approaches are usually determined by the application, radar system parameters, and quantity of training data. 
In \cite{gurbuz2015}, Gurbuz et al. study the effect of various micro-Doppler features, such as transmit frequency, range and Doppler resolution, antenna-target geometry, signal-to-noise ratio and dwell-time. 
While classifier performance on hand-crafted features undoubtedly depend on these parameters, the performance of more advanced feature extraction methods will also vary based on the size of the training set and distributional variation between the train and test data. 
For example, statistical variation within a class can depend on different measurement environments and subjects, which was not addressed in \cite{gurbuz2015}, but explored here by means of features based on singular value decomposition (SVD) or principal component analysis (PCA) of preprocessed spectrograms.

Recently, deep learning architectures, such as CNNs have also been used for micro-Doppler applications \cite{Brooks2018,Huizing2019,Abdulatif2018,Zhu2020}.
By design, CNNs are invariant to shifts in the two-dimensional image, obtained by using a combination of additive group convolutions and pooling operations.
While relatively simple CNNs consisting of only three to five convolutional layers have been shown to be successful, deeper and more complex network approaches have also been investigated \cite{Seyfioglu2019}.
Furthermore, more advanced deep learning approaches have also been taken to compensate for limited training data, such as generating data using generative adversarial networks and transfer learning \cite{Park2016,Tran2020}.

However, there are other works that indicate that the use of CNNs may not be optimal for this application.
Firstly, it is a topic of current research to investigate in what applications CNNs are most successful, why, and how. 
The state of the art use of CNNs typically involve a network topology that is chosen on ad hoc basis. 
In \cite{Tran2020}, the authors find that after mapping the micro-Doppler spectrograms to a one dimensional vector, a linear SVM outperformed various VGG architectures.
Additionally, unlike CNNs, features defined on a specific manifold through an appropriate kernel function provide a clear advantage in being able to more precisely obtain invariance to well defined nuisance transformations known a priori, based on knowledge of the application and data.
This is complemented by theory that gives insight into generalization performance.  

In \cite{Zhengxin2020}, Zeng, Amin, and Shan have demonstrated that  simpler method like K-nearest neighbors (KNN) based on PCA components of the vectorized spectograms yield results comparable to deep networks.

We will use our techniques both with SVD based features and PCA features (designed to overcome the statistical variations within-class), and compare the results  with those obtained with CNNs. 
Our methods will sometimes outperform CNNs and sometimes not, but will be consistently faster on training.
The main advantage though is that our techniques are well founded in theory while it is a major topic of research as to when, why, for what applications, and how CNNs will give what performance.

We will show in Section~\ref{bhag:main} that the process of empirical risk minimization with our kernel yields an exact reproduction of the label function for the training data.
In this sense, the work is related to several recent works which have observed that it is possible to achieve a zero training error while keeping the test error under control.
In the case of classification problems, Belkin, Hsu, and Mitra \cite{belkin2018overfitting} analyze the ``excess error'' in least square 
fits by piecewise linear interpolants over that obtained by the optimal Bayes' classifier. 
In \cite{poggio2017theory3, poggio2018theory3b}, the question is analyzed from the perspective of the geometry of the error surface with respect to different loss functions near the local extrema.
In particular, it is shown in \cite{poggio2017theory3} that substituting the rectified linear unit (ReLU) activation function by a polynomial approximation exhibits the same behavior as the original network.
In \cite{mhaskar2019analysis} we have analyzed the question from the point of view of approximation theory so as to examine the intrinsic features of the data (rather than focusing on specific training algorithms) that allow this phenomenon.
A crucial role in the proofs of the results in that paper is played by a highly localized kernel.
In this paper, we focus on the properties of a localized kernel in a more general setting of a locally compact metric measure space that allow the results analogous to those in \cite{mhaskar2019analysis}.

\bhag{Technical Background}\label{bhag:techintro}

In this section, we discuss some further technical details to motivate our main results in Section~\ref{bhag:main}. 
In Section~\ref{bhag:kernel}, we describe our constructions for approximation on an unknown sub-manifold of the Euclidean space. 
In Section~\ref{bhag:doppler}, we provide a mathematical model for the Doppler radar measurements, leading to the need for examinining our results in Section~\ref{bhag:kernel} to the case of a sub-manifold of a known manifold, namely, the Grassmann manifold.
In Section~\ref{bhag:chellappa}, we describe the ideas in \cite{chellappa} that leads to the Grassaman manifold again from a different perspective of time series classification.

\subsection{Approximation on Manifolds}\label{bhag:kernel}

In this section, we describe our construction for a direct method for function approximation on a sub-manifold of a Euclidean space, without resorting to a two step procedure involving the estimation of quantities related to the manifold such as coordinate charts or eigen-decomposition of the Laplace-Beltrami operator.
Our construction involves a localized kernel based on  Hermite polynomials.

These are defined by the Rodrigues' formula (cf. \cite[Eqns.~(5.5.3), (5.5.1)]{szego}
$$
h_k(x)=\frac{(-1)^k}{\pi^{1/4}2^{k/2}\sqrt{k!}}\exp(x^2)\frac{d^k}{dx^k}\exp(-x^2).
$$
However for relatively small values of $k$ and $|x|<\sqrt{2k}$, they are most efficiently computed using the recurrence relations:
\bea\label{recurrence}
h_k(x)&:=&\sqrt{\frac{2}{k}}xh_{k-1}(x)-\sqrt{\frac{k-1}{k}}h_{k-2}(x), \qquad k=2,3,\cdots,
\nonumber\\
&&h_0(x):=\pi^{-1/4},\ h_1(x):=\sqrt{2}\pi^{-1/4}x.
\eea
We write 
\be\label{uni_psi_def}
\psi_k(x):=h_k(x)\exp(-x^2/2), \qquad x\in\RR,\ k\in\ZZ_+.
\ee 
The functions $\{\psi_k\}_{k=0}^\infty$ are an orthonormal set with respect to the Lebesgue measure on $\RR$.
In the sequel, we fix an infinitely differentiable function $H :[0,\infty)\to [0,1]$, such that $H(t)=1$ if $0\le t\le 1/2$, and $H(t)=0$ if $t\ge 1$.
Further, let $Q\ge q\ge 1$ be fixed integers.
We define for $x\in\RR$, $m\in\ZZ_+$:
\be\label{fastproj}
\mathcal{P}_{m,q}(x):=\begin{cases}
	\disp\pi^{-1/4} (-1)^m\frac{\sqrt{(2m)!}}{2^m m!}\psi_{2m}(x), &\mbox{ if $q=1$,}\\[2ex]
	\disp \frac{1}{\pi^{(2q-1)/4}\Gamma((q-1)/2)}\sum_{\ell=0}^m (-1)^\ell\frac{\Gamma((q-1)/2+m-\ell)}{(m-\ell)!}  \frac{\sqrt{(2\ell)!}}{2^\ell \ell!}\psi_{2\ell}(x), &\mbox{ if $q\ge 2$,}
\end{cases}
\ee
and the kernel $\widetilde{\Phi}_{N,q}$ for $x\in\RR$, $N>0$ by
\be\label{fastkerndef}
\begin{aligned}
	\widetilde{\Phi}_{N,q}(x)&:=\sum_{m=0}^{\lfloor N^2/2\rfloor}H\left(\frac{\sqrt{2m}}{N}\right)\mathcal{P}_{m,q}(x)\\
	&=\frac{1}{\pi^{q/2}\Gamma((q-1)/2)}\sum_{\ell=0}^{\lfloor N^2/2\rfloor} \left\{\sum_{m=\ell}^{\lfloor N^2/2\rfloor} H\left(\frac{\sqrt{2m}}{N}\right)\frac{\Gamma((q-1)/2+m-\ell)}{(m-\ell)!}\right\}\psi_{2\ell}(0)  \psi_{2\ell}(x).
\end{aligned}
\ee
The kernel $\widetilde{\Phi}_{N,q}(x)$ can be computed easily and efficiently using the second equation in \eqref{fastkerndef}, and the Clenshaw algorithm \cite[p.~79]{gautschibk}.

Next, let $\XX$ be a $q$ dimensional, compact, connected, orientable, sub-manifold of $\RR^Q$, $\rho$ be the geodesic distance on $\XX$, $\mu^*$ be its volume element.
We will abbreviate our notation and write $\Phi_N^M(\x)=\widetilde{\Phi}_{N,q}(|\x|_{2,Q})$, $\x\in\XX$ \footnote{We note that when $N=1$, $\Phi_N^M(\x)$ reduces to just a constant multiple of $\exp(-|\x|^2/2)$.}. \\

Let $C(\XX)$ denote the class of all continuous real valued functions on $\XX$ equipped with the supremum norm $\|f\|_\infty:=\max_{x\in\XX}|f(x)|$. 
Using the kernel $\Phi_N^M$, we define an integral operator on $C(\XX)$, analogous to a convolution operator,  by
\be\label{eq:manifold_op_def}
\sigma_{N}^M(f)(\x):=\int_\XX \Phi_N^M(\x-\y)f(\y)d\mu^*(\y), \qquad N>0.
\ee
(In the notation of \cite{mhaskar2019deep}, $\sigma_N^M(f)=\sigma_{N,1}(\XX;f)$.)
We are interested in approximation of Lipschitz continuous functions on $\XX$; i.e., functions $f\in C(\XX)$ for which
\be\label{eq:manifoldlipnormdef} 
\|f\|_{\mbox{Lip}}:=\sup_{\x,\x'\in\XX\atop \x\not=\x'}\frac{|f(\x)-f(\x')|}{\rho(\x,\x')} <\infty.
\ee

A consequence of  \cite[Theorem~8.1]{mhaskar2019deep} is the following Theorem~\ref{theo:main}.
Before stating this theorem, we state first a convention regarding constants.\\

\noindent\textbf{Constant convention}\\[1ex]
\textit{In the sequel, the notation $A\ls B$ (equivalently, $B\gs A$) will mean that $A\le cB$ for some positive constant $c$ depending only on fixed objects under discussion such as $\XX$, $\rho$, $\mu^*$, $q$, $Q$, the parameter $S$ to be introduced later, and the manifold/spaces/measures/distances to be described later.
	The notation $A\sim B$ means $A\ls B$ and $B\ls A$. In particular, the constants do not depend upon the target function, the points at which the approximation is desired or on which it is based, and the index $N$ of the family of kernels.}\\

\begin{theorem}\label{theo:main}
	Let $f\in \mathsf{Lip}(\XX)$. 
	For $N\ge 1$,  we have
	\be\label{fundaapprox}
	\|f-\sigma_N^M(f)\|_\infty \ls \frac{\|f\|_{\mbox{Lip}}}{N}.
	\ee 
\end{theorem}  

An important ingredient in the proof of Theorem~\ref{theo:main} is  the following 
proposition proved in  \cite[Corollary~6.1]{mhaskar2019deep}, and reformulated using the fact that $|\x-\y|_{2,Q}\sim \rho(\x,\y)$ for $\x,\y\in\XX$ \cite[ Corollary~8.1]{mhaskar2019deep}.

\begin{prop}\label{prop:tildephi}
	Let $S>q$. The   kernel $\Phi_N^M$ defined in \eref{fastkerndef} satisfies each of the following properties.
	\be\label{tildephiloc}
	|\Phi_N^M(\x-\y)| \ls \frac{N^q}{\max(1,(N\rho(\x,\y))^S)}, \qquad \x\in\RR^q,
	\ee
	\be\label{tildephibern}
	|\Phi_N^M(\bs 0)| \sim N^q, \quad |\Phi_N^M(\x)|\ls N^q, \quad |\Phi_N^M(\x)-\Phi_N^M(\y)| \ls N^{q+1}\left||\x|_{2,Q}-|\y|_{2,Q}\right|, \qquad \x,\y\in\RR^Q.
	\ee
\end{prop}

\noindent\textbf{Motivation 1:} \textit{In practice, of course, one needs to use a discretization of the integral in \eqref{eq:manifold_op_def}. 
	Obviously, an accurate discretization would require the knowledge of the values of $f$ at a sufficiently large number of points in $\XX$. 
	In micro-Doppler radar applications, one has only a small amount of data, so that the above theorem cannot be used effectively in a direct manner. }\\[1ex]

\subsection{Micro-Doppler Radar Signals}\label{bhag:doppler}

In this section we denote the radar location by $\x \in \RR^3$ and the target location by $\y \in \RR^3$. 
We assume that at $t=0$ the object is initially located at $\y_0 \in \RR^3$.
During the collection period $t \in [0,T]$, the target undergoes motion defined by a displacement with velocity $\v \in \RR^3$ and a rotation given by the matrix $\mathbf{R}(t)$.
Thus, at time $t$ the location of the object is
\begin{equation}\label{eq:target_position}
	\y(t) = \mathbf{R}(t)\y_0 + \v t.  
\end{equation}
Assuming the radar transmits in free space, and that there are only a finite number $K$ of isotropic scattering locations\footnote{Isotropic scattering means the signal is reflected in all directions, a reasonable assumption for small scattering locations on the human hand.},
we may model the received signal as (cf. \eqref{eq:pdef})
\begin{equation}\label{eq:received_signal_2}
	s(t) = \sum_{k=1}^{K} p( t - \| \x - \mathbf{R}_k(t)\y_{0,k} - \v_k t  \|_2 )\rho_k  + n(t),  
\end{equation}
where for $k=1,\cdots,K$,  each $\y_{0,k}$ represents the initial scattering locations of the object that contribute to the signal measured at the receiver, $\mathbf{R}_k(t)$ and $\mathbf{v}_k$ are the corresponding rotations and velocities, $\rho_k \in \RR^{+}$ is the amplitude of the scatterer, and $n(t)$ represents noise, clutter, and small amplitude interfering signals when the isotropic point scattering assumption does not hold precisely. 
Typically, $p(t)$ is a frequency modulation continuous (FMCW), defined as
\begin{equation}\label{eq:pdef}
	p(t) = e^{i(\omega_c t + \frac{1}{2}\alpha t^2)},
\end{equation}
where $\omega_c$ is the carrier frequency and $\alpha$ is the chirp rate. 
The data set introduced in \cite{Ritchie2020} and used in our experiments, were collected with a FMCW waveform.

For gesture recognition $K$ will be small since there are only a small number of dominant scattering locations. Thus, the received signal $s(t)$ will lie in a low-dimensional subspace given by the span of the individual received signals for each scatterer, of maximum dimension $K$. This is due to the fact that the scattered signals are approximately orthogonal for large $\omega_c$ used for mico-Doppler radars, and motivates our following choice of features.

Generally, micro-Doppler signal classification is accomplished using the spectrogram of \eqref{eq:received_signal_2}, defined as
\begin{equation}\label{eq:stft_data}
	D(\omega,\tau) = \int_\RR s(t+\tau) w(t) e^{-i \omega (t+\tau)} dt,
\end{equation}
where $w$ is a smooth window function.
In practice, \eqref{eq:received_signal_2} will be sampled to form a discrete time-series, and one uses the discrete version of \eqref{eq:stft_data}:
\begin{equation}\label{eq:discrete_spec}
	\mathbf{D} = \mathbf{F} \W \mathbf{S},
\end{equation}  
where $\mathbf{F}$ is a partial Fourier matrix, $\W = \text{diag}(\w)$ is a diagonal matrix consisting of the elements of $\w$, and $\mathbf{S} = [ \mathbf{s}_1, \mathbf{s}_2, \dots, \mathbf{s}_N ]$.
The columns of $\mathbf{D}$ correspond to the discrete Doppler frequencies $\omega_{-M/2},\dots,\omega_{M/2}$, and the rows correspond to the discrete time points $\tau_0,\dots,\tau_N$.

Let the singular value decomposition of $\mathbf{D}$ be given by 
\begin{equation}\label{eq:data_svd}
	\mathbf{D} = \U \boldsymbol{\Sigma} \V^H,
\end{equation} 
where $\U$, $\V$ are unitary matrices, and $\boldsymbol{\Sigma}$ is a diagonal matrix consisting of the singular values.
From \eqref{eq:data_svd} we can define a number of distance functions based on using left singular vectors and singular values, each of which defines a sub-manifold of the data.

\subsection{Representation of a General Time Series}\label{bhag:chellappa}
The autoregressive-moving-average (ARMA) is a well-known dynamic model for time series data that
parametrises a signal $f(t)$ by the equations

\begin{equation}
	f(t) = Cz(t) + w(t), \quad w(t) \sim \mathcal{N}(0, R)
\end{equation}
\begin{equation}
	z(t + 1) = Az(t) + v(t), \quad v(t) \sim \mathcal{N}(0, Q)
\end{equation}

where $z \in \mathbb{R}^d$ is the hidden state vector, $f : \mathbb{R} \rightarrow
\mathbb{R}^p$,  $d \leq p$ is the hidden state dimension, and $N(0,\sigma)$ denotes a normal distribution with mean zero and standard deviation $\sigma$ \cite{chellappa}.
There are widely-used closed form solutions for estimating the parameters $A$ and $C$ in terms of the singular value decomposition $[f(1),\cdots, f(\tau)]=U\Sigma V^T$, namely,
\be\label{eq:arma_model}
C=U, \quad A=\Sigma V^TD_1V(V^TD_2V)^{-1}\Sigma^{-1}, 
\ee
where
$$
D_1=\begin{pmatrix} 0&0\\I_{\tau-1}&0\end{pmatrix}, \qquad D_2=\begin{pmatrix} I_{\tau-1}&0\\ 0&0\end{pmatrix}.
$$
It can be
shown \cite{chellappa} that the expected observation sequence is given by

\begin{equation}\label{eq:observation}
	O_\infty = \mathbb{E}\left[\begin{pmatrix} f(0) \\ f(1) \\ f(2) \\ \vdots
	\end{pmatrix}\right] = \begin{bmatrix} C \\ CA \\ CA^2 \\ \vdots \end{bmatrix} z(0)
\end{equation}
Thus, by truncating the matrix $O_\infty$ up to $m$-th block for some $m$, one can represent the time series $f$ as point on the Grassmann manifold $\Gc(d,mp)$.

We consider time-frequency images, in this case spectrograms generated from the radar data as multivariate timeseries, and similarly for the frames of the video data. In the former case, one way to view the spectrogram is as a multivariate time series, where the columns are multi-dimensional vectors formed by the taking the Discrete Fourier Transform of the time segment, the collection of these vectors implies the spectrogram is a multivariate time series.
As a consequence we can represent these spectrograms using an orthonormal basis corresponding to a point on the Grassmann manifold. 
Independently, as explained in Section~\ref{bhag:doppler}, the physics of radar imaging leads to treating the left singular vectors of the spectrograms as feature vectors, which in turn, are points on a Grassmann manifold.
These features capture the frequency content of the micro-Doppler signature while imposing invariance with respect to the underlying time-series, which can better represent the discriminant attributes common to the entire data set.
For example, we expect multiple samples corresponding to the same gesture can be represented by a collection of orthonormal basis matrices that are close under a particular metric. 
Furthermore, the subspace corresponding to a data sample can be viewed as subspace of a larger subspace spanned by the all the samples of the class.  \\

\noindent\textbf{Motivation 2: }\textit{Clearly, the set of time series of interest, such as the low rank representations of the micro-Doppler radar signals, is only an unknown subset of the Grassmann manifold, rather than the entire manifold. 
	Assuming that this subset is a sub-manifold, it is interesting to examine the span of $\{\widetilde{\Phi}_{N,q}(\rho(\circ,\y_j)\}_{j=1}^M$, where $\rho$ is the geodesic distance on the Grassmann manifold. }

\bhag{Theoretical Results}\label{bhag:main}

Our goal in this section is to formulate results for approximation based on a small amount of data. 
We will examine two tools for this purpose: empirical risk minimization, and theoretical least square loss. 
In each case, we will obtain pointwise error estimates which indicate how the error deteriorates as the input variable moves away from the training data.
Our estimates will give a deeper insight into which features of the data give rise to what accuracy \textit{intrinsically}, as opposed to many results in the literature that depend upon the method used to solve the optimization problems involved.
In view of the motivation given in Section~\ref{bhag:chellappa}, we will state our results in greater abstraction than in the context of the motivation given in Section~\ref{bhag:kernel}.

Let $\XX$ be a locally compact metric measure space, with metric $\rho$, and a distinguished measure $\mu^*$.
As before, $C(\XX)$ denotes the space of bounded and uniformly continuous real valued functions on $\XX$, equipped with the supremum norm: $\|f\|_\infty =\sup_{x\in\XX}|f(x)|$. 
The class $\mbox{Lip}(\XX)$ of Lipschitz functions comprises $f\in C(\XX)$ for which
\be\label{eq:lipnormdef} 
\|f\|_{\mbox{Lip}}:=\sup_{\x,\x'\in\XX\atop \x\not=\x'}\frac{|f(\x)-f(\x')|}{\rho(\x,\x')} <\infty.
\ee

We recall the constant convention from Section~\ref{bhag:kernel}. 

\begin{definition}\label{def:lockern}
	Let  $S>q\ge 1$ be  integers. A family of kernels $\{\Phi_N :\XX\times \XX\to\RR\}$ is called \textbf{$(q,S)$-localized} if each $\Phi_N$ is symmetric, and
	\be\label{eq:genlocalest}
	|\Phi_N(x, y)|\ls \frac{N^q}{\max(1, (N\rho(x,y))^S)},
	\ee
	where the constant involved in $\ls$ may depend upon $q$ and $S$ but not on $N$, $x$, or $y$.
	With an abuse of terminology we will say that $\Phi_N$ is $(q,S)$-localized.
\end{definition}
\begin{rem}\label{rem:examples}
	{\rm Kernels satisfying \eqref{eq:genlocalest} are known in many contexts. 
		The kernel denoted by $\Phi_N^M$ in Section~\ref{bhag:kernel} is one example. 
		Many other constructions in different contexts are discussed in \cite{mhaskar2020kernel} and references therein. \qed}
\end{rem}

Let $M\ge 1$ be an integer, $\C=\{y_1,\cdots,y_M\}\subset\XX$. 
We are interested in studying approximation of Lipschitz functions on $\XX$ from the space
\be\label{eq:gentransspace}
\mathcal{V}(\C)=\mathsf{span}\{\Phi_N(\circ,y_j)\}_{j=1}^M.
\ee

\begin{rem}\label{rem:kernelvsspace}
	{\rm
		If $\Phi_N(x,y)$ admits a Mercer expansion of the form $\sum \lambda_k \phi_k(x)\phi_k(y)$, and the set $\C$ is sufficiently dense in $\XX$, then in a number of cases as in \cite{mhaskar2020kernel}, the space $\mathcal{V}(\C)$ is the same as $\mathsf{span}\{\phi_k : \lambda_k\ls N\}$. 
		However, our interest in this paper is when $\C$ has a large minimal separation rather than being dense.
		\qed}
\end{rem}

We will study the behavior of approximations defined by the solutions of two minimization problems:\\[1ex]

\noindent The \textbf{empirical risk}  minimizer is defined
to be
\be\label{eq:empiricaldef}
P_E(\mathcal{V}(\C);f):=\argmin_{P\in \mathcal{V}(\C)}\sum_{k=1}^M \left|f(y_j)-P(y_j)\right|^2.
\ee

\noindent If $\tau$ is a probability measure on $\XX$,
\textbf{theoretical least square loss } minimizer is defined by
\be\label{eq:theoreticallossdef}
P_T(\tau, \mathcal{V}(\C);f):=\argmin_{P\in\mathcal{V}(\C)}\int_\XX \left|f(y)-P(y)\right|^2d\tau(y).
\ee

We need to assume some further conditions on the kernels $\{\Phi_N\}$. 

\begin{definition}\label{def:admissibledef}
	Let $S>q\ge 1$ be integers. A family of kernels $\{\Phi_N\}$ is called \textbf{admissible} if it is $(q,S)$-localized and satisfies each of the following properties.
	\be\label{eq:ondiagonal}
	|\Phi_N(y,y)|\sim N^q, \qquad y\in\XX.
	\ee
	\be\label{eq:kernellip}
	\begin{aligned}
		|\Phi_N(x,y)-\Phi_N(x',y)|&\ls & N^{q+1}\rho(x,x'), \qquad x,x',y\in \XX,\\
		|\Phi_N(x,y)-\Phi_N(x,y')|&\ls & N^{q+1}\rho(y,y'), \qquad x,y,y'\in \XX.
	\end{aligned}
	\ee
	With an abuse of terminology as before, we refer to any member $\Phi_N$ of an admissible family as an admissible kernel.
\end{definition}
\begin{rem}\label{rem:admissible}
	{\rm
		The kernel denoted by $\Phi_N^M$ in Section~\ref{bhag:kernel} is one example of admissible kernels (cf. Proposition~\ref{prop:tildephi}). 
		Several other examples are available in the literature, for example, based on eigenfunctions of the Laplace-Beltrami operator on manifolds (\cite{frankbern, modlpmz}).
		\qed}
\end{rem}

\begin{rem}\label{rem:qSextra}
{\rm In the sequel, $q$ and $S$ will be treated as fixed parameters. 
All constants may depend upon these.
Moreover, we will omit their mention from the notation; e.g., say that $\Phi_N$ is localized rather than $(q,S)$-localized.
\qed}
\end{rem}

In the study of approximation properties of both the empirical risk minimizer and theoretical least squares loss minimizer, a crucial role is played by the minimal separation among the points in $\C$, defined by
the minimal separation among the points  defined by
\be\label{eq:minsep}
\eta(\C)=\min_{1\le j\not=k\le M}\rho(y_j,y_k),
\ee

In order to state our main theorems, we need to make some further assumptions.

For $x\in\XX$, $r>0$, we write
\be\label{balldef}
\BB(x,r)=\{y\in\XX: \rho(x,y)\le r\}, \qquad \Delta(x,r)=\XX\setminus \BB(x,r).
\ee

We assume that there exists $q\ge 0$ such that 
\be\label{eq:ballmeasurecond}
\mu^*(\BB(x,r))=\mu^*\left(\{y\in\XX : \rho(x,y)<r\}\right)\ls r^q, \qquad x\in \XX, \ r>0,
\ee
and
\be\label{eq:ballmeasurelow}
\mu^*(\BB(x,r))\gs r^q, \qquad x\in\XX, \ 0<r\le 1.
\ee

It is shown in \cite{mhaskar2020kernel} that in the case when $\XX$ is a manifold as described earlier, the condition \eqref{eq:ballmeasurelow} is satisfied in fact for all $r\le \mathsf{diam}(\XX)$.
The condition \eqref{eq:ballmeasurecond} is satisfied  in many cases when $\XX$ is a smooth compact connected $q$-dimensional manifold, $\rho$ is the Riemannian metric, and $\mu^*$ denotes the volume measure.

Our first theorem derives pointwise error estimates for the empirical risk minimizer..

\begin{theorem}\label{theo:empirical}
	Let $\{\Phi_N\}$ be a family of admissible kernels, and \eqref{eq:ballmeasurecond} and \eqref{eq:ballmeasurelow} be satisfied.
	Let $x\in\XX$,   $\delta(x)=\min_{1\le j\le M}\rho(x,y_j)$, and $\tilde{\eta}=\min(1,\eta(\C))$. Let $F\in \mbox{{\rm Lip}}(\XX)$. There exists a constant $C>0$ independent of $F$, $N$, and $\C$ such that if  $N\ge C\eta(\C)^{-1}$,
	then there exists a unique empiricial risk minimizer $P_E(\mathcal{V}(\C);F)$ that satisfies
	\be\label{eq:riskmin}
	P_E(\mathcal{V}(\C);F)(y_j)=F(y_j), \qquad j=1,\cdots, M.
	\ee
	We have the following error estimates. \\
	If $\delta(x) > \tilde{\eta}/3$ then
	\be\label{eq:genapproxaway}
	\left|P_E(\mathcal{V}(\C);F)(x)\right| \ls \frac{\|F\|_\infty}{(N\tilde{\eta})^{S}}.
	\ee
	If $\delta(x)\le \tilde{\eta}/3$ then
	\be\label{eq:genapproxbdnear}
	\left|P_E(\mathcal{V}(\C);F)(x)-F(x)\right|\ls (N+\|f\|_{\mbox{Lip}})\delta(x) +\|F\|_\infty(N\tilde{\eta})^{q-S}.
	\ee
	
	Here, the constants involved are independent of $F$ and $x$.
\end{theorem}

To describe the analogue of Theorem~\ref{theo:empirical} for the minimizer of the theoretical least square loss, we need some further notation relevant to the normal equations for the minimization problem.
We define
\be\label{eq:psindef}
\Psi_N(x,y)=\int_\XX \Phi_N(x,z)\Phi_N(y,z)d\tau(z), \qquad N>0, \ x, y\in \XX,
\ee
and for $f\in C(\XX)$,
\be\label{eq:gensigmadef}
\sigma_N(f)(x)=\int_\XX f(y)\Phi_n(x,y)d\mu^*(z).
\ee
We note that if $\tau$ is absolutely continuous with respect to $\mu^*$ with the Radon-Nikodym derivative $f_0$, then the normal equations for the theoretical least squares loss minimization are
\be\label{eq:normeqn}
\sum_{\ell=1}^M a_\ell\Psi_N(y_\ell,y_j)=\sigma_N(ff_0)(y_j), \qquad j=1,\cdots, M.
\ee

\begin{theorem}\label{theo:theoretical}
	We assume the set up as in Theorem~\ref{theo:empirical}. Let $\tau$ be absolutely continuous with respect to $\mu^*$ with $d\tau =f_0d\mu^*$ for some $f_0\in C(\XX)$, $f_0(z)\ge \mathfrak{m}(f_0)>0$ for all $z\in\XX$. 
	There exists a constant $C>0$ independent of $F$, $N$, and $\C$ such that if  $N\ge C\eta(\C)^{-1}$,
	then there exists a unique theoretical least square loss risk minimizer $P_T(\tau,\mathcal{V}(\C);F)$ that satisfies the following error estimates.\\
	If $\delta(x) > \tilde{\eta}/3$ then
	\be\label{eq:theo_away}
	\left|P_E(\mathcal{V}(\C);F)(x)\right| \ls \frac{\|f_0F\|_\infty}{(N\tilde{\eta})^{S}}.
	\ee
	If $\delta(x)=\rho(x,y_\ell)\le \tilde{\eta}/3$ then
	\be\label{eq:theoretical_err_est}
	\begin{aligned}
		\Bigg|P_T(\tau,\mathcal{V}(\C);F)(x)&-\left.\frac{\Phi_N(x,x)}{\Psi_N(y_\ell,y_\ell)}f_0(x)F(x)\right|\\
		&\ls (N\eta)^{q-S}\|f_0F\|_\infty +\|f_0F-\sigma_N(f_0F)\|_\infty+(N+\|f_0F\|_{\mbox{Lip}})\delta(x).
	\end{aligned}
	\ee
	Here, the constants involved are independent of $F$ and $x$.
\end{theorem}

\begin{rem}\label{rem:sigmapprox}
	{\rm
		In the case of the kernels $\Phi_N^M$, the quantity $\|f_0F-\sigma_N(f_0F)\|_\infty$ is estimated by Theorem~\ref{theo:main}. 
		Similar estimates exist also in the other examples described in \cite{mhaskar2020kernel} with other conditions on $f_0$ and $F$.
		\qed}
\end{rem}

\begin{rem}\label{rem:approxrem}
	{\rm
		We note that \eqref{eq:genapproxaway} and \eqref{eq:theo_away} show that one does not expect a good approximation to $F$ as $N\to\infty$ in the region away from the points $\C$.
		On the other hand, \eqref{eq:genapproxbdnear} and \eqref{eq:theoretical_err_est} show that the approximation actually deteriorates if $N\to\infty$. 
		The balance of the various terms on the right hand side of these estimates gives us some insight into the choice of $N$  as well as how the approximation error deteriorates as $x$ is further and further away from $\C$.
		This phenomenon is caused by the localization of the kernels, resulting in overfitting the data if $N$ is very large. 
For example, in our experiments in Section~\ref{bhag:svd_performance_summary}, we will need to use $N=1$; the higher values of $N$ led to overfitting.
\qed}
\end{rem}

\bhag{Experimental Results}\label{bhag:experiments}

In this section, we describe the use of the localized kernels discussed in Sections~\ref{bhag:techintro} and
\ref{bhag:main} in the context of gesture recognition with micro-Doppler radar readings (Section~\ref{bhag:dopplersect}), and to demonstrate the fact that the theory is applicable agnostically to the domain knowledge, also in the context of a small video data set (Section~\ref{bhag:naokisect}).
As remarked before, we use the kernel $\widetilde{\Phi}_{N,q}(\rho(x,y))$ with different distances $\rho(x,y)$, corresponding to different types of manifolds, which depend on the feature representation.

\subsection{Gesture Recognition}\label{bhag:dopplersect}

In the case of micro-Doppler gesture recognition experiments, we use two kinds of features   of the measured spectrogram: features based on the singular value decomposition of the spectrogram viewed as a matrix, and features based on the PCA components of a vectorized version of the spectrogram.
For comparison, we include the results based on two CNNs, which worked with the same  preprocessed data as the other methods.
We do not consider hand-crafted features here, because both the SVD and PCA features yield substantially better results.

In view of the results described in more detail in \cite{Konstantin2018}, the first set of features we considered are the first $r$ columns of $\mathbf{U}$ to be a feature vector $\mathbf{U}_r$ for the signal $\mathbf{D}$ in \eqref{eq:stft_data} (cf. \eqref{eq:data_svd}). 
We note that this feature vector is a point on the Grassmann manifold $\Gc(r,N)$, with geodesic approximations that can be calculated efficiently.
As a variant of this set of features, we relax the assumption that subspaces of gesture spectrograms are suitable for discrimination, and consider incorporating the singular values by defining distances using both $\U$ and $\boldsymbol{\Sigma}$, which we refer to as the SVD manifold. 

Lastly, we relax any assumptions about the discriminative capability of the SVD components between individual samples, and instead use Principle Component Analysis (PCA) to calculate a common low-dimensional subspace to project the data on. 
Our methods are guided by the paper \cite{Zhengxin2020} by Zeng, Amin, and Shan.
Precise definitions of the metrics and kernels used in these approaches are listed in Table \ref{table:kerntable}.

We describe the data set in Section~\ref{bhag:gesturedata}.
The data involves 4 gestures made by 6 persons.
We focus first on recognition of gestures independently of the person making the gesture.
The pre-processing step is described in Section~\ref{bhag:preprocess}.
Section~\ref{bhag:kernelsused} describes the various kernels which will be used in our experiments.
The methodology of the experiments as well as the results are described in Section~\ref{bhag:svd_performance_summary} for the case of SVD features, Section~\ref{bhag:pca_performance_summary} for the case of PCA features, and \ref{bhag:cnn_accuracy} for the case of CNNs.
The effect of choosing the dimension of the manifold is discussed in Section~\ref{bhag:dimension}.
The effect of choosing various percentage of the training data is described in Section~\ref{bhag:trainsize}.
In Section~\ref{bhag:crosssubject}, we examine the results of training on all the gestures made by 5 of the subjects, and predicting the results for the corresponding gesture by the 6th subject.

\subsubsection{Data Set}\label{bhag:gesturedata}
We use the publicly available micro-Doppler data set \emph{Dop-NET} distributed by University College London Radar group \cite{Ritchie2020}. 
The data set consists of complex spectrograms formed using motion compensated, measured micro-Doppler signatures of five people performing four different hand gestures: snap, wave, pinch, click.
The data is obtained using a linear FMCW at $24$ GHz center and $750$ MHz bandwidth. 
While a training and test set are provided, no ground truth values are provided for test set.
Thus, we split the provided training set into a training and test set using a $80/20\%$ split.
This train/test split is used for all our experiments, except those in Section \ref{bhag:trainsize} where we investigate the effect of the size of the training set.
In each of our experiments we report the mean and variance of $5$ trials of independent sampling of the train/test set.

In Figure \ref{fig:representative_images}, an example spectrogram is shown for each class.
Since the time duration of each gesture varies, so does the number of columns in each spectrogram, which can vary by up to an order of magnitude in terms of number of samples.
Furthermore, we note that there are visible similarities between the classes; for example, the pinch, swipe and click gesture are both visually similar in terms of frequency and time support, while the wave gesture lasts for longer time duration and demonstrates oscillatory behavior.     
Clearly, these similarities make this classification problem particularly difficult, and motivate the use of features that are robust to the within class variation of time and Doppler frequencies.

\begin{figure}[ht]
	\centering
	\begin{subfigure}{.45\textwidth}
		\centering
		\includegraphics[width=.8\linewidth]{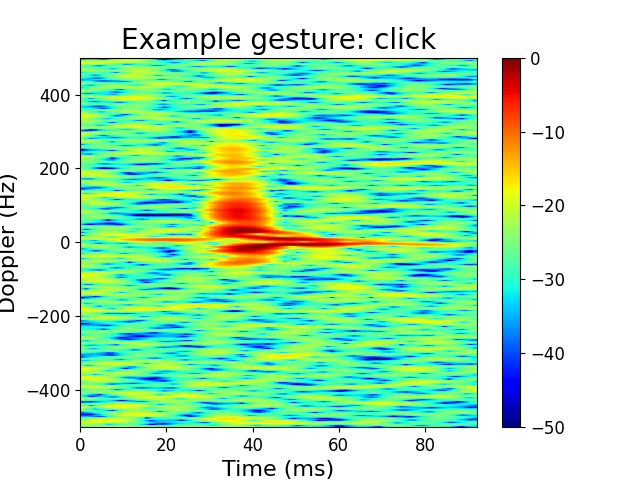}
		\caption{}
		\label{fig:click}
	\end{subfigure}
	\centering
	\begin{subfigure}{.45\textwidth}
		\centering
		\includegraphics[width=.8\linewidth]{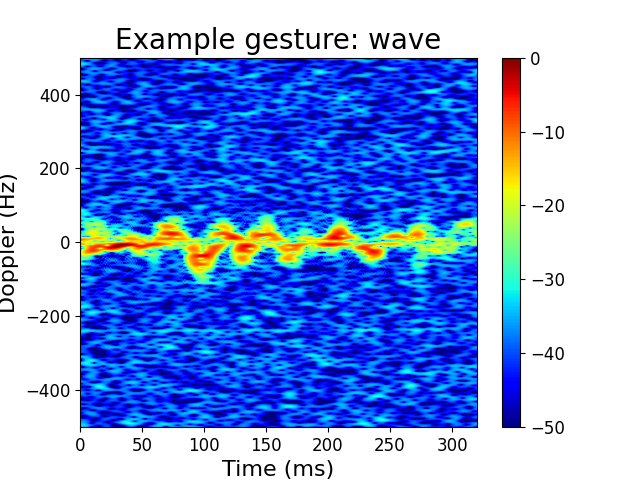}
		\caption{}
		\label{fig:wave}
	\end{subfigure}\centering
	
	\begin{subfigure}{.45\textwidth}
		\centering
		\includegraphics[width=.8\linewidth]{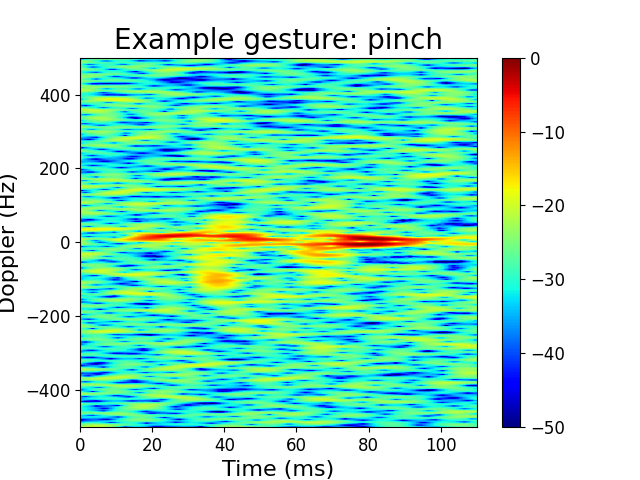}
		\caption{}
		\label{fig:pinch}
	\end{subfigure}
	\begin{subfigure}{.45\textwidth}
		\centering
		\includegraphics[width=.8\linewidth]{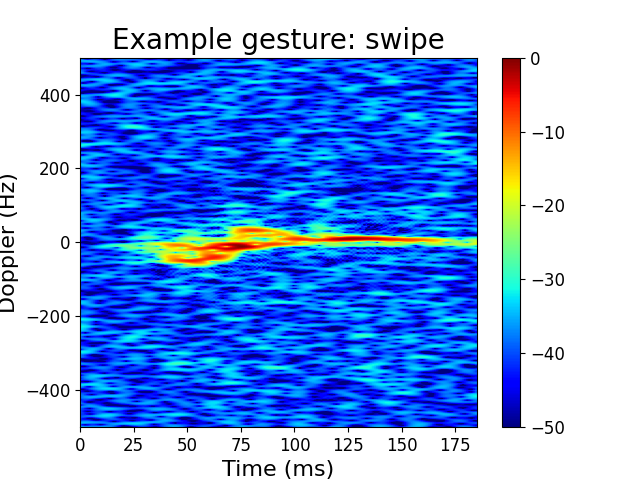}
		\caption{}
		\label{fig:swipe}
	\end{subfigure}
    \caption{A data sample selected from each of the four gestures, provided in the Dop-Net data set. Note that while the figures are scaled to be the same size, the time axis varies between different gestures. \subref{fig:click}) Click. \subref{fig:wave}) Wave.  \subref{fig:pinch}) Pinch. \subref{fig:swipe}) Swipe.}
	\label{fig:representative_images}
\end{figure}

\subsubsection{Data Preprocessing}\label{bhag:preprocess}

Prior to classification, we perform data preprocessing to reduce noise and clutter; then normalize to account for variation in the dynamic range.
To suppress the effects of noise and clutter we investigated the use of various binary thresholding algorithms to segment the signal from background, including the envelope detection approach \cite{Zhengxin2020}.
It was determined that Yen's threshold method worked best, and was used for all of our experiments \cite{Yen1995}.
We can explicitly summarize this process as
\begin{equation}\label{eq:thresholded_data}
	\tilde{X}_k = \Tc\left\{ 20 \text{log}_{10} \left( |X_k| \right) \right\}
\end{equation}
where $X_k$ is a sample from the data set, $\Tc$ is thresholding function described in \cite{Yen1995}.

Following thresholding, we normalize the data to account for the variation in the dynamic range.
We consider two approaches to this normalization which are both used in the experiments.
The first approach is to normalize the data point $\tilde{X}_k$ such that each element is in the range $[0,1]$. 
Our second approach is to convert the threshold spectrogram to a binary representation where each element is in the set $\{0,1\}$.

The SVD features are based on $r$ left singular vectors and corresponding singular values in the decomposition  
$\tilde{X}_k=U_k\Sigma_k V_k^T$, $k=1,2, \cdots, r$, for various values of $r$.

To obtain the PCA features, the spectrograms are zero padded to the same size, then vectorized.
Zero padding was also used for the inputs to the CNN. 
The PCA is made on the entire training data, and the projection of any datum on the first $r$ PCA vectors was taken as the representation of that datum in this feature space.

\subsubsection{Kernels}\label{bhag:kernelsused}

The set of kernels used in our experiments together with the values of the parameters are summarized in Table~\ref{table:kerntable}. The notation $U$, $\Sigma$  refers to  the left singular vectors and singular values of the data corresponding to the $r$ largest singular values obtained from the spectrograms as described in Section~\ref{bhag:preprocess}, and $\|\cdot\|_F$ denotes the Frobenius norm.  
The three kernels considered are the \textit{Grassmann kernel}, which uses the projection distance, an approximation of the geodesic on the Grassmann manifold \cite{Ye2016}. 
The \textit{Laplace kernel} and \textit{Gaussian kernel} use a distance function that incorporates both the singular vectors and singular values of the spectrograms.
We note that the Gaussian kernel corresponds to the case where $N=1$ and $q=1$ in \eqref{fastkerndef}.
The general kernel as defined in \eqref{fastkerndef} is referred to as the \textit{localized kernel}.

\begin{table}[ht]
	\begin{center}
		
		\begin{tabular}{|c|c|c|}
			\hline
			Kernel Type & Expression & Parameters \\ 
			\hline
			Grassmann kernel & $\exp(-\gamma (r - \|U_1^TU_2\|_F^2))$ & $\gamma = 0.2$ \\
			\hline
            Laplace kernel & $\exp(-\alpha \|U_1 - U_2\|_F - \beta \|\Sigma_1 - \Sigma_2\|_F)$
                           & $\alpha = 0.2$, $\beta = 0.0042$ \\
			\hline
			Gaussian kernel & $\exp(-\alpha\|U_1 - U_2\|^2_F - \beta \|\Sigma_1 - \Sigma_2\|^2_F)$
                            & $\alpha = 0.2$, $\beta = 0.12$ \\
            \hline
            Localized kernel (\eqref{fastkerndef}) & $\widetilde{\Phi}_{N, q}(\gamma x)$ & $\gamma = 0.8$, $q = 2$ ($n = 16$) or $q = 18$ ($n = 64$) \\
			\hline
		\end{tabular}
	\end{center}
	\caption{The definition of various kernels used in our experiments.}
	\label{table:kerntable}
\end{table}

\subsubsection{Performance Accuracy for SVD Features}\label{bhag:svd_performance_summary}

We evaluate the performance of the proposed algorithms in terms of classification accuracy and measured computational time, as these are major considerations in transitioning technology.
The results are organized by the type of preprocessing and feature type, where
preprocessing refers to whether the data is normalized or binary, and feature type refers to methods that use components of the SVD decomposition as features or PCA projections.
In this section, we focus on the SVD features. We use values of $r$ given in Table \ref{fig:hidden_dim_table}, which were found to perform best for each SVD-based method in testing.

\begin{table}[ht]
    \centering
    \begin{tabular}{|l|r|}
        \hline
        Method & $r$ \\
        \hline \hline
        Gaussian SVD SVM (binary) & 4 \\
        Grassmann SVD SVM (binary) & 5 \\
        Laplace SVD SVM (binary) & 11 \\
        Gaussian SVD SVM (normalized) & 3 \\
        Grassmann SVD SVM (normalized) & 7 \\
        Laplace SVD SVM (normalized) & 19 \\
        \hline
    \end{tabular}
    \caption{The values of $r$ (number of selected singular vectors and values) used for each of the SVD SVM methods.}
    \label{fig:hidden_dim_table}
\end{table}

To perform classification using the kernels we trained a support vector machine classifier, which was implemented in Python using SciKit-Learn \cite{pedregosa2011scikit}.
In the discussion of these methods and the corresponding results we use the notation ** SVD SVM to denote the SVM based on SVD features and using the kernel **. 
For example, Grassman SVD SVM refers to the SVM based on SVD features and the Grassman kernel.

\begin{table}[ht]
	\centering
	\begin{subfigure}{\textwidth}
		\centering
		\begin{tabular}{|l|rrr|}
			\hline
            Approach (Binary) & Average Accuracy (\%) & Test Time (s) & Train Time (s) \\ \hline \hline
            Gaussian SVD SVM & 89.94$\pm$0.0 & 17.8$\pm$0.0 & 48.15$\pm$0.27 \\ 
            Grassmann SVD SVM & 83.16$\pm$0.02 & 19.58$\pm$0.14 & 54.15$\pm$0.92 \\ 
            Laplace SVD SVM & 87.02$\pm$0.01 & 21.47$\pm$0.01 & 61.8$\pm$0.04 \\ \hline
		\end{tabular} \vspace{.1in} 
		\caption{}
		\label{fig:binary_SVD_table}
	\end{subfigure}
	\centering
	\begin{subfigure}{\textwidth}
		\centering
		\begin{tabular}{|l|rrr|}
			\hline
            Approach (Normalized) & Average Accuracy (\%) & Test Time (s) & Train Time (s) \\ \hline \hline
            Gaussian SVD SVM & 88.58$\pm$0.0 & 17.9$\pm$0.0 & 48.1$\pm$0.01 \\ 
            Grassmann SVD SVM & 86.65$\pm$0.01 & 20.92$\pm$0.01 & 58.85$\pm$0.67 \\ 
            Laplace SVD SVM & 85.13$\pm$0.02 & 25.12$\pm$0.02 & 76.41$\pm$1.29 \\ \hline
		\end{tabular} \vspace{.1in}
		\caption{}
		\label{fig:normalized_SVD_table}
	\end{subfigure}\centering
    \caption{ The performance accuracy, training time, and testing time with variance (written
    as accuracy$\pm$variance) for the kernel methods using SVD features  are presented for the different preprocessing approaches. \subref{fig:binary_SVD_table}) Results for binary preprocessing.
    \subref{fig:normalized_SVD_table}) Results for normalized preprocessing. }
	\label{fig:overall_SVD_result_table}
\end{table}

The results of these experiments are displayed in Figure \ref{fig:overall_SVD_result_table}. 
The kernel methods based on the binary features performed in the range of $83.16\%$ to $89.94\%$, with the Grassmann SVD SVM performing worst, and the Gaussian SVD SVM the best.

When the normalized samples are used, the Gaussian and Lacplace SVD SVMs see a performance decrease over binary samples by $1.36\%$ and $0.89\%$, respectively. 
The opposite is true for the Grassmann SVD SVM which increases by $3.49\%$.   
We believe the reduction in performance when using SVD derived features may be a result of an undesired invariance to magnitude and time variation, a consequence of using the singular vectors as a feature, and suggests that the singular values provide more useful information in the case of binary data. 
It is possible that the left singular vectors fail to represent certain non-stationary statistical behavior, as the time ordering is lost when calculating the singular vectors.
We leave a detailed analysis of these feature types for future work.

 \subsubsection{Performance Accuracy for PCA Features}\label{bhag:pca_performance_summary}

In this section, we focus on the results based on the PCA features described in Section~\ref{bhag:preprocess}.

We used a $5$-nearest neighbor algorithm (KNN) as well as SVM trained with the localized kernel $\tilde{\Phi}_{N,q}$ defined in \eqref{fastkerndef}, which we denote as PCA LocSVM*, where * denotes the degree of the polynomial in \eqref{fastkerndef}; i.e., $2\lfloor N^2/2\rfloor$. 
We note that   an expression of the form $\sum_j a_j\tilde{\Phi}_{N,q}(|\x-\y_j|_Q)$ resembles an average of $a_j$ at $\x$, weighted depending upon the distance of $\x$ from the point $\y_j$. 
Here $Q$ is the number of singular values used in our computation; i.e., $Q = 30$. The optimal value of $q$ is shown empirically to be 2 for degree $n = 16$ and $18$ for degree $n = 64$. This suggests that in the 30-dimensional feature space,  the data ``lives'' on an $18$-dimensional sub-manifold.
In view of the localization property \eqref{tildephiloc}, many terms in this expression will be close to $0$, and there are only a few terms involved in this average for any $\x$ depending upon its distance from various points. 
So in some sense, this is also a nearest neighbor estimate. 
However, rather than prescribing how many nearest neighbors to use, we allow the parameter $N$ to control this number effectively. 
Thus,  how many neighbors are counted effectively depend upon how many points $\y_j$ are there in a neighborhood of $\x$, the radius of this neighborhood controlled by $N$.
We note that the weights involved in this averaging are not all positive. 
This may look unusual, but is an essential requirement to get to good theoretical results on the accuracy of approximation.

The binary results are shown in Table \ref{fig:binary_PCA_table}, and the normalized results in Table~\ref{fig:normalized_PCA_table}.

\begin{table}[ht]
	\centering
	\begin{subfigure}{\textwidth}
		\centering
		\begin{tabular}{|l|rrr|}
			\hline
			Approach (Binary) & Average Accuracy (\%) & Test Time (s) &  Train Time (s) \\ \hline \hline
            PCA KNN & 95.81$\pm$0.0 & 0.42$\pm$0.0 & 22.49$\pm$0.23 \\ 
            PCA LocSVM16 & 94.83$\pm$0.0 & 0.72$\pm$0.0 & 24.89$\pm$0.05 \\ 
            PCA LocSVM64 & 95.81$\pm$0.0 & 5.69$\pm$0.0 & 62.27$\pm$0.14 \\
            CNN1 & 90.06$\pm$0.01 & 12.69$\pm$0.01 & 14770.88$\pm$749.08 \\ 
            CNN2 & 88.38$\pm$0.01 & 4.79$\pm$0.01 & 2162.59$\pm$952.77 \\ \hline
		\end{tabular} \vspace{.1in} 
		\caption{}
		\label{fig:binary_PCA_table}
	\end{subfigure}
	\centering
	\begin{subfigure}{\textwidth}
		\centering
		\begin{tabular}{|l|rrr|}
			\hline
			Approach (Normalized) & Average Accuracy (\%) & Test Time (s) &  Train Time (s) \\ \hline \hline
            PCA KNN & 95.36$\pm$0.0 & 0.53$\pm$0.02 & 23.51$\pm$0.57 \\ 
            PCA LocSVM16 & 93.72$\pm$0.01 & 0.73$\pm$0.0 & 25.61$\pm$0.1 \\ 
            PCA LocSVM64 & 96.18$\pm$0.0 & 5.97$\pm$0.15 & 64.53$\pm$1.5 \\
            CNN1 & 92.85$\pm$0.01 & 12.7$\pm$0.01 & 14746.82$\pm$342.77 \\ 
            CNN2 & 91.38$\pm$0.01 & 4.75$\pm$0.0 & 2133.72$\pm$307.46 \\ \hline
		\end{tabular} \vspace{.1in}
		\caption{}
		\label{fig:normalized_PCA_table}
	\end{subfigure}\centering
    \caption{ The performance accuracy, training time, and testing time with variance (written
    as accuracy$\pm$variance) for the kernel methods using the different PCA feature based methods considered and different preprocessing
    schemes. Results are  presented also for the two CNNs (Section~\ref{bhag:cnn_accuracy}). \subref{fig:binary_PCA_table}) Results for binary preprocessing.
    \subref{fig:normalized_PCA_table}) Results for normalized preprocessing. }
	\label{fig:overall_PCA_result_table}
\end{table}

In the binary case, the PCA KNN performed identically with PCA LocSVM64 obtaining an accuracy score of $95.81\%$ with very small variance. 
The smaller degree polynomial kernel PCA LocSVM16 performed slightly worse by $0.98\%$.
There is $2.48\%$ variation when using the normalized data, shown in Table \ref{fig:normalized_PCA_table}. The best performing method is PCA LocSVM64 achieving $96.81\%$ accuracy.
The PCA KNN achieved $95.36 \%$ which is $0.45\%$ below the binary normalized.
Based on these results, we conclude the overall performance between binary and normalized preprocessing techniques are nearly equivalent, with the binary preprocessing reducing variance between methods.
The fact that the high order localized kernel based SVM method was able to accurately classify the PCA transformed data to a level comparable to a $5$ nearest-neighbor classifier suggests that in the reduced space the data is dense, with the classes well separated. This property allows the kernel to be highly localized, but generalize well.

In terms of computational time, PCA KNN and LocSVM16 performed similarly in both train and testing, performing inference on the entire test set in less than a second, where as the PCA LocSVM64 was an order of magnitude slower taking $5.69$ seconds on the test set, and three times slower taking $62.27$ seconds to train.
We note however that our localized kernel is a very newly introduced kernel, and our computation is a brute force evaluation.
It is a subject of future research to develop fast algorithms involving Hermite polynomial expansions based at the so-called scattered data\footnote{This term refers to data points whose location is not prescribed in a specific way, for example, zeros of Hermite polynomials.}, and in particular, the fast evaluation of the kernels taking into account their localization.

\subsubsection{Performance Accuracy for CNNs}\label{bhag:cnn_accuracy}

Another popular approach to micro-Doppler gesture recognition are CNNs.
We compare the proposed methods with two different mid-size CNN architectures used for gesture recognition \cite{Kim2016,Kulhandjian2019}, which we refer to as CNN1 and CNN2, respectively.
The input to both of these networks are the preprocessed spectrograms (\ref{eq:thresholded_data}) as  described in section \ref{bhag:preprocess}, which were also used to extract the SVD and PCA features.
This ensures that every approach uses the same preprocessed data, which is further zero padded in the case of the CNNs and PCA features.
Both networks consist of $4$ convolutional layers with ReLU activation followed by $2:1$ max pooling, and the prediction is made using a single dense layer.
For CNN1, the first layer consisted of $20$ channels of $8 \times 8$ filters, the second layer is $10$ channels of $16\times 16$ filters, and the final layer is $5$ channels of $32\times 32$ filters. 
The three layers of CNN2 each consist of $5$ channels with $5\times 5$, $4\times 4$, $2\times2$ sized filters. 
Both networks were implemented using PyTorch \cite{NEURIPS2019_9015} and trained using cross-entropy loss function with the adaptive moment estimation (Adam) optimizer for $100$ epochs, at which point the models had converged.
A mini-batch size of $16$ and learning rate $0.001$ was used. 
We also experimented using networks with increased feature dimension, but this led to over-fitting due to the small size of the data set.
Therefore, pursuing deeper networks was unlikely to yield improved results.
Though it is worth noting that due to the small training set, CNN1 and CNN2 are still significantly over parameterized. 

The results are given together with those for the  PCA features in Table~\ref{fig:overall_PCA_result_table}.
With both the binary and normalized spectrograms, the PCA KNN as well as the methods PCA LocSVM16 and PCA LocSVM64 based on our localized kernels clearly outperform both the CNNs both in terms of accuracy and in terms of training time.
The testing time for CNNs is comparable with that of PCA LocSVM64.

In comparison with methods based on SVD features,  we  observe that the Gaussian SVM surpasses CNN2 by $1.56\%$ in the case of binary preprocessing.
When normalized features are used, CNN1 and CNN2  see accuracy improvements of $1.79\%$ and $3.00\%$, respectively.
The improvement in CNN performance is clearly a result of the increased information representation in the dynamic range of the spectrograms, where as only the support is available in the binary case.

In terms of time complexity in comparison with SVD based features,  CNN2 performs inference on the entire data set in $4.7-4.8$ seconds, but at the cost of $35$ minutes to train.
CNN1 is a larger network and saw inference time increase to $12.7$ seconds and more than 4 hours to train. 
In contrast, the kernel methods inference time on the test ranged from $12.69$ to $25.12$ seconds and training times in the range $48.15$ to $76.41$ seconds.

While the CNNs maintain a slight classification accuracy and inference time increase over the SVD kernel methods,  it takes significantly larger amount of time to train. While the CNNs appear to be faster than the kernel methods, the variation in inference time is most likely a result of the fact that Pytorch is a C++ library wrapped in Python, whereas Sci-Kit Learn is pure python, which is much slower.  

Overall, we found that using PCA transformed data was a better feature selection than those based on the singular vectors and values of the individual data points.
In fact, it outperformed both CNNs best performance by an average of $3\%$ accuracy.
Based on performance, we conclude that the LocSVM64 kernel method performs very closely to PCA KNN and best overall.

\subsubsection{Sensitivity to Feature Dimension}\label{bhag:dimension}

In this section, we investigate how well the proposed kernels perform when the total number of left singular vectors $r$ changes.
We carry out this experiment by using only the top $r$ singular values and corresponding singular vectors of the data. 
This can be viewed as function of information captured by the features used to construct the kernels defined in Table \ref{table:kerntable}. 
Ideally, $r$ should be chosen large enough to retain a sufficient amount of information required for optimal classifier performance.
Since the received signal spans a low-dimensional subspace, mentioned in Section \ref{bhag:doppler}, we expect that the different methods will plateau after a certain dimension due to the maximal amount of information being captured.
   
\begin{figure}[ht]
	\centering
	\begin{subfigure}{.45\linewidth}
		\centering
		\includegraphics[width=.9\linewidth]{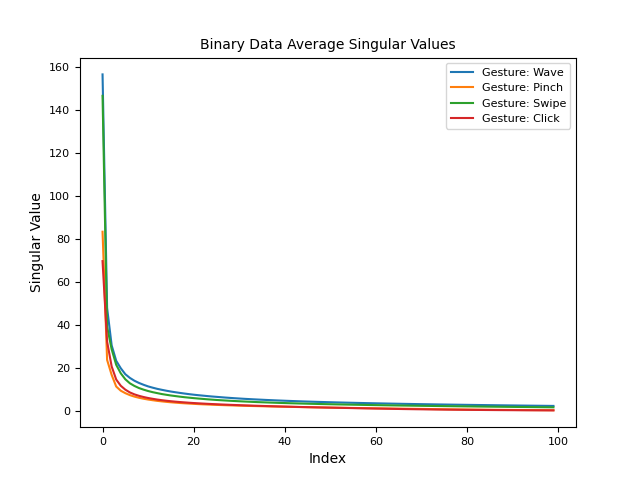}
		\caption{}
		\label{fig:bin_singular_value_dist}
	\end{subfigure}\centering
	\centering
	\begin{subfigure}{.45\linewidth}
		\centering
		\includegraphics[width=.9\linewidth]{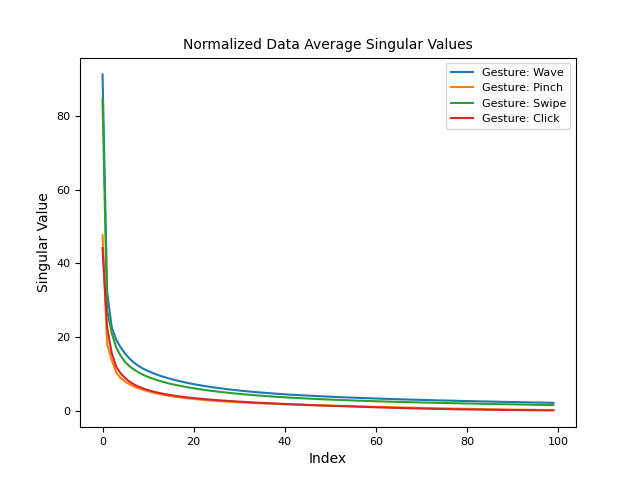}
		\caption{}
		\label{fig:norm_singular_value_dist}
	\end{subfigure}
	\caption{Plot of the first 100 average singular values for each gesture. 
\subref{fig:bin_singular_value_dist}) Binary preprocessing.
	\subref{fig:norm_singular_value_dist}) Normalized preprocessing.}
	\label{fig:singular_value_plots}
\end{figure}

In Figure~\ref{fig:singular_value_plots}, we show the singular values averaged over all realizations of each gesture.
Clearly, for each gesture and preprocessing technique there is a sharp drop before $r = 20$ and the rest of the singular values are near zero.
This suggests that the dimension of the subspace features need not be large, this is expected based on the micro-Doppler structure described in Section~\ref{bhag:doppler}.

\begin{figure}[ht]
	\centering

    \begin{subfigure}{.45\linewidth}
        \centering
	    \includegraphics[width=0.9\linewidth]{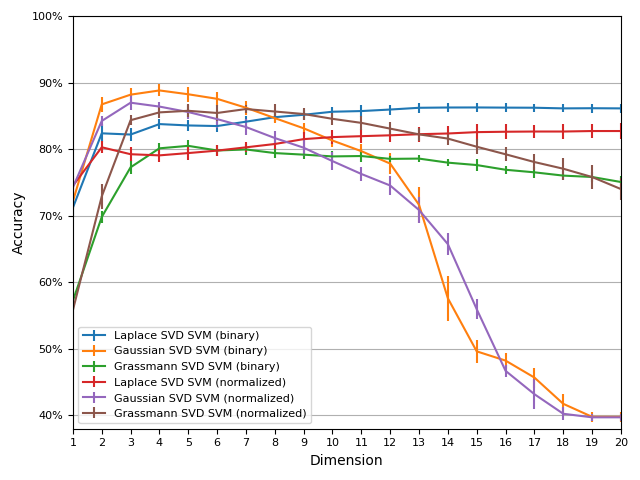}
        \caption{}
        \label{fig:error_vs_dim_svd}
    \end{subfigure}
    \begin{subfigure}{.45\linewidth}
        \centering
        \includegraphics[width=0.9\linewidth]{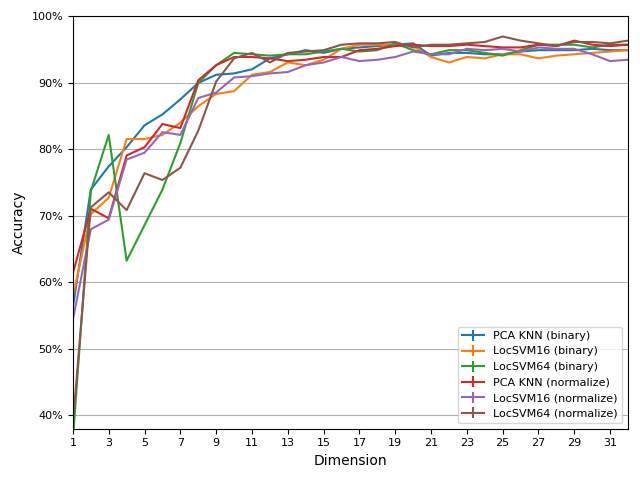}
        \caption{}
        \label{fig:error_vs_dim_pca}
    \end{subfigure}

	\caption{Classification accuracy averaged verse the dimension $r$ of the feature manifold for
    each method considered on both binary and normalized preprocessing schemes.
    \subref{fig:error_vs_dim_svd}) SVD-based methods.  \subref{fig:error_vs_dim_pca}) PCA-based
    methods. In this plot the $y$-axis is $40\%$ to $100\%$.}
	\label{fig:error_vs_dim}
\end{figure}

In Figure \ref{fig:error_vs_dim_svd} we plot the average classification error of the methods based on SVD features for $r = 1,\dots,20$ and the different preprocessing techniques. 
For $r=1$ all the SVD kernels out-perform the Grassmann kernels by approximately $18\%$.
This discrepancy in performance reduces as $r$ increases to $3-4$, at which point the performance of the Gaussian kernel begins to drop, while the Grassmann and Laplace kernels perform relatively constant with a few percent variance.   
At this time, we are not able to explain the drop in performance of the Gaussian kernels, which we find interesting. 

The proposed approaches using PCA features, are shown in Figure \ref{fig:error_vs_dim_pca}. 
Here, the number $r$ of PCA components is the dimension of the ambient space, denoted by  $Q$ in \eqref{fastkerndef}.
 For the LocSVM tests, we adjust the manifold dimension $q$ to ensure that it is no greater than the feature space dimension $r$. 
 Specifically, for values of $r$ at or above $18$, $q$ is set to $18$ for LocSVM64. For values of $r$ less than $18$, $q$ is set to $r$. 
 Since LocSVM16 was found to perform best for $q = 2$, it did not need adjusting to remain below $r$, except in the case when $r=q=1$.
The performance of all methods and preprocessing techniques perform similarly, achieving accuracy in the mid $90\%$ for sufficiently large $r$.
Choice of preprocessing seemingly has little effect on performance, observed in the prior experiments as well. 
There is a also a large drop off in performance as the dimension reduces from $r=10$.
The KNN PCA and LocSVM64 degrade similarly, with classification performance collectively falling by $30\%$.
The LocSVM16 performs worse, with accuracy dropping over $70\%$ as $r$ reduces from $10$ to $1$.
This drop is expected since for very small $r$ only a small amount of discriminitive information of the data is preserved.

\subsubsection{Robustness to Limited Training Data}\label{bhag:trainsize}

In many radio-frequency machine learning applications training data is difficult to collect and will likely be the limiting factor of the model's performance.
In the case of gesture recognition, not only is cost and effort high, but the data collection involves human subjects \cite{Gurbuz2020_book}. 
Thus, it is important to develop an approach that is robust to limited data. 
As a result of the complexity of data collection, what is considered a large data set in micro-Doppler recognition problems will pale in comparison to those used as benchmarks in the Deep Learning community, such as  ImageNet used in object recognition \cite{deng2009imagenet}.
For example, the Dop-net data set consists of $2433$ training samples and probably less testing samples\footnote{Since the test set is not publicly released, we are assuming the test set is small compared to the training set, which is typical ML practice.}, while ImageNet has $14$ million.

\begin{figure}[ht]
	\centering
    \begin{subfigure}{.45\linewidth}
	    \includegraphics[width=.9\linewidth]{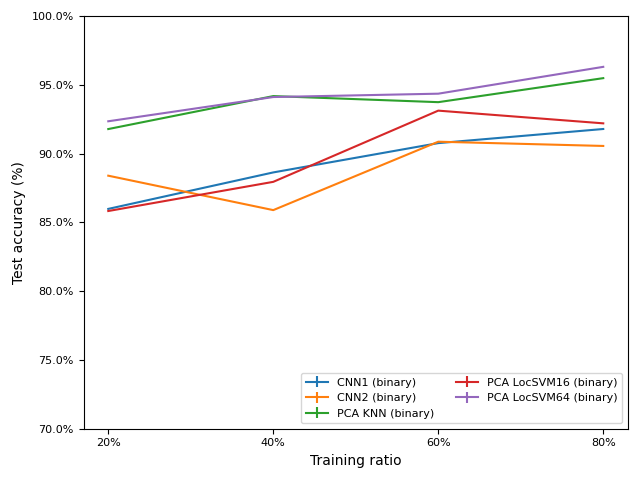}
        \caption{}
        \label{fig:test_error_vs_training_size_pca_binary}
    \end{subfigure}
    \begin{subfigure}{.45\linewidth}
	    \includegraphics[width=.9\linewidth]{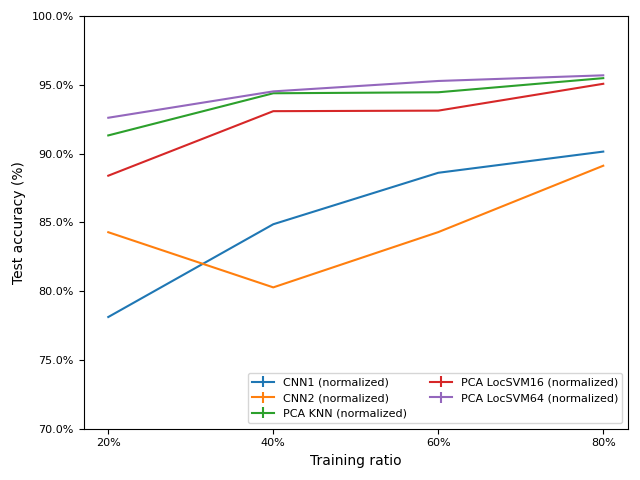}
        \caption{}
        \label{fig:test_error_vs_training_size_pca_normalized}
    \end{subfigure}
	\caption{Test error vs. training size for each PCA-based method and preprocessing approach considered. \subref{fig:test_error_vs_training_size_pca_binary}) Binary preprocessing.  \subref{fig:test_error_vs_training_size_pca_normalized}) Normalized preprocessing. In this plot the $y$-axis is $70\%$ to $100\%$.}
    \label{fig:test_error_vs_training_size_PCA}
\end{figure}

\begin{figure}[ht]
	\centering
    \begin{subfigure}{.45\linewidth}
	    \includegraphics[width=.9\linewidth]{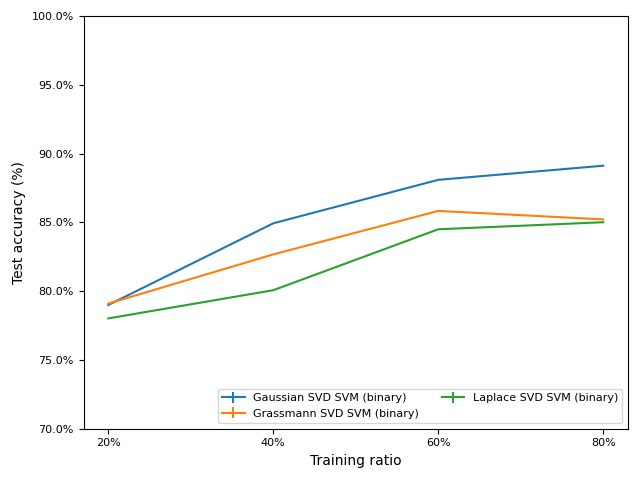}
        \caption{}
        \label{fig:test_error_vs_training_size_raw_binary}
    \end{subfigure}
    \begin{subfigure}{.45\linewidth}
	    \includegraphics[width=.9\linewidth]{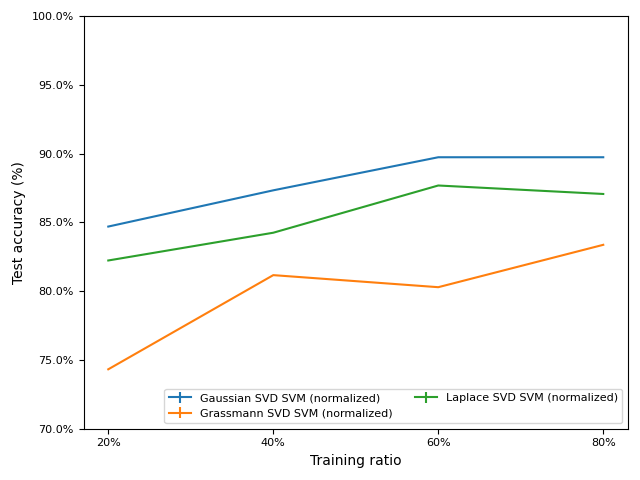}
        \caption{}
        \label{fig:test_error_vs_training_size_raw_normalized}
    \end{subfigure}
	\caption{Test error vs. training size for each  and preprocessing approach based on SVD.
     \subref{fig:test_error_vs_training_size_raw_binary}) Binary preprocessing is used.
    \subref{fig:test_error_vs_training_size_raw_normalized}) Normalized preprocessing is used.  In this plot the $y$-axis is $70\%$ to $100\%$.}
    \label{fig:test_error_vs_training_size_SVD}
\end{figure}

We evaluate each approach on $\{20,40,60,80\}\%$ of the training data for five trials and average the results.
Figures \ref{fig:test_error_vs_training_size_PCA} and \ref{fig:test_error_vs_training_size_SVD} show the classification accuracy for the methods based on SVD and PCA features for different training set sizes, respectively.
The total amount of data used for training varied between $487$ and $1947$ samples.  

Figure \ref{fig:test_error_vs_training_size_pca_binary} shows the results for the PCA based features with binary preprocessing. 
As expected, the classification accuracy increases with the size of the data set, this is also the case for normalized data, shown in Figure \ref{fig:test_error_vs_training_size_pca_normalized}.
For both the binary and normalized data preprocessing the PCA-KNN and PCA LocSVM64 demonstrate equivalent performance over the range of training data.
In the binary case, all the methods and training sizes classification accuracy are in the range $85.82\%$ to $96.30\%$, and $88.39\%$ to $95.65\%$ for normalized preprocessing. 
It is notable that in the case of normalized preprocessing, PCA LocSVM64 does marginally outperform PCA KNN by $0.14\%$-$1.28\%$, while the PCA LocSVM16 under performs in every case, and by a larger margin when binary preprocessing is used.

Ignoring the under performance of PCA LocSVM16, the performance range reduces to $91.78\%$ to $96.30\%$ for binary preprocessing and up to $91.32\%$ and $95.69\%$ for normalized preprocessing.
Considering the tight performance range and inconsistency of any of the methods to consistently outperform, we conclude that the methods are robust to limiting the amount of training data. 
While this does not necessarily guarantee robustness when the data set size gets significantly larger, in the case of PCA LocSVM64 and PCA KNN we expect this observed robustness will hold.
The reason is that when using PCA features the data manifold in this case allows a greater separation among the classses than that based on SVD features. Therefore, the PCA LocSVM64 can better approximate the data with fewer samples. 
In theory, of course, the main problem is to find the right feature space; the straightforward approximation procedure developed in \cite{mhaskar2019deep} obviates the need to do any training at all.

The SVD features result in more variation of classification accuracy and are shown in Figure \ref{fig:test_error_vs_training_size_SVD}.
With the exception of CNN2, which has an accuracy drop at $40\%$ train/test ratio, all the methods demonstrate a clear monotonic increase in accuracy with increasing training data.
The SVD features in the binary preprocessing are shown in Figure \ref{fig:test_error_vs_training_size_raw_binary}, where the range of classification accuracy is $78.02\%$ to $90.14\%$.
The normalized preprocessed SVD feature performance are shown in Figure \ref{fig:test_error_vs_training_size_raw_normalized}, with classification accuracy in the range of $78.02\%$ to $91.79\%$.
Overall, there was is a clear reduction in performance when using SVD features opposed to PCA. 
The PCA LocSVM64 outperformed the SVD features by $6.16\%$ and $3.86\%$ for the binary and normalized preprocessing methods.

In the binary case, CNN1 and CNN2 outperformed the methods using SVD features, except for CNN2 when the train/test ratio was $40\%$. 
Following the CNNs in classification performance was the Gaussian SVD SVM, Laplace SVD SVM and the Grassmann SVD SVM, listed in order of descending performance.
With normalized preprocessing the Gaussian SVM outperforms all other methods, except for CNN1 with $80\%$ train/test ratio, in which case CNN1 only outperforms by $0.41\%$.

Clearly, the proposed manifold learning approach appears robust to using just a few hundred samples.
We attribute this property to a number of factors, first of which is micro-Doppler radar data can be very sensitive to multiple scattering and environmental effects. 
The choice of preprocessing clearly plays a role in suppressing these nuisance variables, which is then improved by the fact that using only $r$ singular vectors as features further reduces uninformative information.

Furthermore, the $r$ dimensional subspace representation is able to capture the time-varying properties of the data, which correspond to important discriminative features in the original time-series and improves generalization.
This is further motivated, by noting that the relationship between the subspace and received radar signals, which span a low-dimensional subspace.
To understand this, we observe that in \eqref{eq:discrete_spec} the rank of the spectrogram is determined only by the rank of the matrix $\mathbf{S}$.
Therefore, since the micro-Doppler signature consists of a linear combination of delayed and Doppler shifted signals from the dominant scattering points, the rank of the resulting matrix $\mathbf{S}$ is small, thus, it is well captured with only a small number of singular vectors. 
Furthermore, the proposed kernel provides good localization at data points on the manifold, which we suspect helps to isolate nuisance variables in the data.

It is very surprising that the CNN continues to perform well with limited training samples. 
While this may seem impressive, it is difficult to generalize this claim to all micro-Doppler data sets as the literature shows varying performance when similar models are used on different data sets. 
We think this may be a result of the fact that the training set is already so small to begin with, so  training either the full data set or only $20\%$ does not make a significant difference in performance due to the very large over-parameterization of the neural networks.
This is in contrast to our kernel approach where the model does not over-parameterize the data.
Furthermore, this peculiar performance may also be a result of robustness added by the preprocessing used prior to input into the CNN.

\subsubsection{Cross-Subject Generalization}\label{bhag:crosssubject}

Generalization is an important aspect in any machine learning task and seeks to ensure that the algorithm will perform well on out-of-sample data. 
In micro-Doppler radar, there are a number of factors that can cause variation of the in- and out-sample distributions, like environmental variations leading to changes in the electromagnetic (EM) wave propagation medium, changes in material from which the EM waves scatter off, and non-stationary clutter and multiple scattering effects. 
For example, in synthetic aperture radar target classification, it has been observed that deep learning classifiers will fail to generalize from synthetic data to measured data.
This is attributed to the fact that neural networks tend to learn background surface statistics instead of true features of the object of interest \cite{jo2017}.
In this case, the problem is that simulating the scattering off foliage is computationally intractable.
These types of challenges have also been observed in classification of communication signals, where algorithms will fail to generalize to data collected on different days.

\begin{figure}[!ht]
	\centering
	\begin{subfigure}{.45\linewidth}
		\centering
		\includegraphics[width=.9\linewidth]{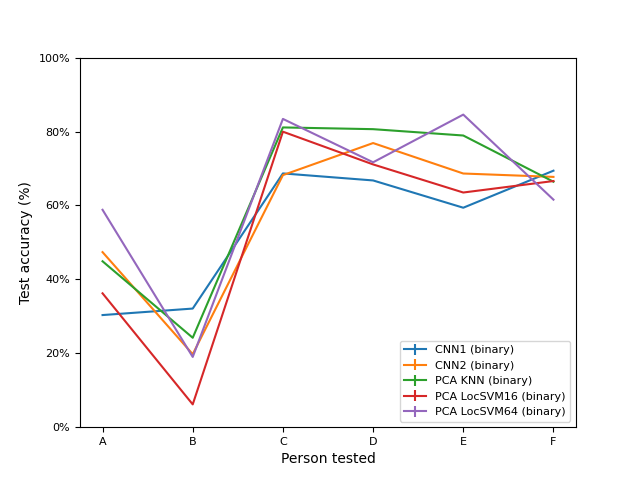}
		\caption{}
		\label{fig:diff_pers_pca_binary}
	\end{subfigure}
	\centering
	\begin{subfigure}{.45\linewidth}
		\centering
		\includegraphics[width=.9\linewidth]{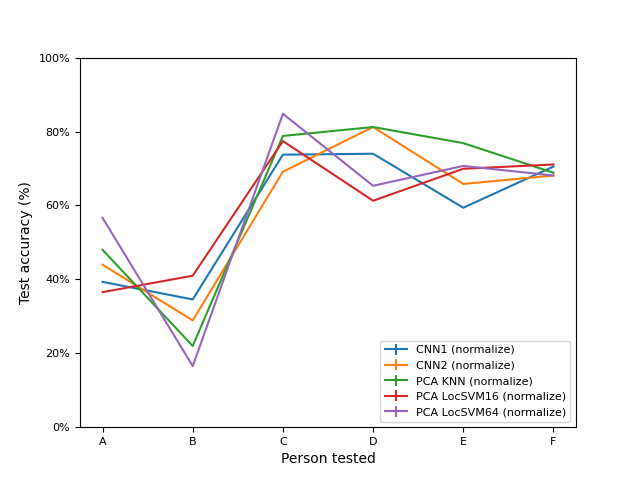}
		\caption{}
		\label{fig:diff_pers_pca_norm}
	\end{subfigure}
    \centering
	\begin{subfigure}{.45\linewidth}
		\centering
		\includegraphics[width=.9\linewidth]{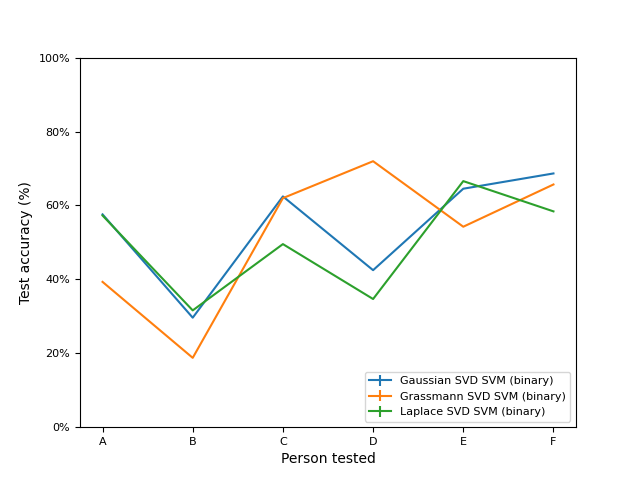}
		\caption{}
		\label{fig:diff_pers_raw_binary}
	\end{subfigure}
    \centering
	\begin{subfigure}{.45\linewidth}
		\centering
		\includegraphics[width=.9\linewidth]{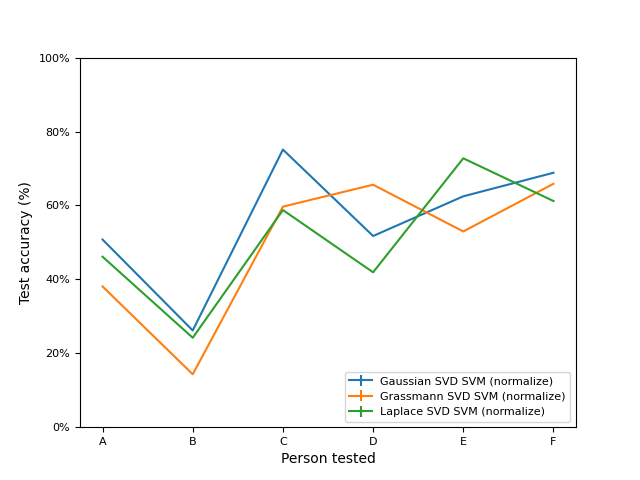}
		\caption{}
		\label{fig:diff_pers_raw_norm}
	\end{subfigure}
    \centering
    \caption{Test accuracy for each algorithm and preprocessing approach when the methods are
    trained on data from five people and tested on the sixth. \subref{fig:diff_pers_pca_binary})
    PCA based features with binary preprocessing. \subref{fig:diff_pers_pca_norm}) PCA based features with normalized preprocessing. \subref{fig:diff_pers_raw_binary}) SVD features with
    binary preprocessing. \subref{fig:diff_pers_raw_norm}) SVD features with normalized
    preprocessing.}
	\label{fig:test_error_vs_different_persons}
\end{figure}

Since the data set consists of collections from six different individuals, it is an interesting question as to whether the manifold learning approach is robust to testing on individuals it was not trained on.
To investigate this, we train on five of the six individuals and test on the remaining one.
We present these results in Figure~\ref{fig:test_error_vs_different_persons}.
From these results, it is interesting to see that there is a $40\%$ to $50\%$ variation in performance between different combinations of methods and the person held out of training. This suggests that there must be characteristics between each person not necessarily captured by the others as possible environmental effects are mitigated by the fact that the data is collected in a lab experiment.

For example, in the case of PCA features, all the methods perform most poorly on Person $B$.
In the binary case, PCA LocSVM16 performs worse with accuracy score $6.17\%$ and PCA KNN performs best, scoring $24.20\%$ accuracy.
The performance increased by $20\%$ to $40\%$ in the case of Person $A$ and increased by $40\%$ to $60\%$, in the case of either preprocessing choice.
There is no consistent pattern of out performance by any particular approach, but they do follow the same trend of when they do and do not generalize well.

In the case of SVD features, the overall performance is lower, ranging from $18.77\%$ for the Grassmann SVD SVM when tested on Person $B$, up to $81.21\%$ for CNN2 with normalized preprocessing when tested on person $D$.
Similarly for PCA features, there is no approach that consistently outperforms all other methods for each person.
An interesting behavior is how the Grassmann SVD SVM performs worse than the Gaussian SVD SVM and Laplace SVD SVM on Person $A$ and $B$, but then outperforms on person $C$, and comparably for Persons $D$, $E$ and $F$.

Based on these results it is clear that the choice of preprocessing does not have a meaningful impact in generalization across different persons, thogh there is a slight out performance of normalization preprocessing.
Clearly, for some persons the methods are able to generalize well, whereas in some cases they do not.
We suspect this may have to do with variation in how different persons make specific gestures and/or body size, which can effect the distance and orientation of the scattering centers to the radar.
This could effect both the amplitude and support of the spectrogram and no longer have the same distribution as the training data.

While these results do not provide a clear pattern, it motivates the requirement for a more in depth and controlled study as to what physical effects of the micro-Doppler data are important for achieving good generalization. 
Thus, we conclude that there is relationship between understanding the data from a physics perspective and what are appropriate approaches.
Overall this suggests that optimal performance is as much a function of the ML algorithm, as it is the data collection approach and preprocessing.

\subsection{Performance on Video-Data}\label{bhag:naokisect}

The purpose of this section is demonstrate that our kernel can be used agnostically to the domain knowledge about the data. 
Toward this goal, we examine a small and simple video data set (without the audio component) where the objective is to identify which digit was spoken in the video.
In Section~\ref{bhag:naokidata}, we describe the data set.
Our methodology and results are described in Section~\ref{bhag:methods}. 

\subsubsection{Data Set}\label{bhag:naokidata}
We use the data described in \cite{lieu_signal_2011}.
The data consists of video clips of the same person pronouncing digits from 1-5 in English. Each of the five digits is
recorded ten times, yielding 50 total videos. 
Each video is cropped to a 70 x 55 area around the
mouth, with the position of the nose held constant, and then converted to grayscale. 
The problem is
to read what digit between one and five the subject is pronouncing in each video clip, effectively a
signal classification problem among 5 classes.

\begin{figure}
	\centering
	\includegraphics[width=.35\textwidth]{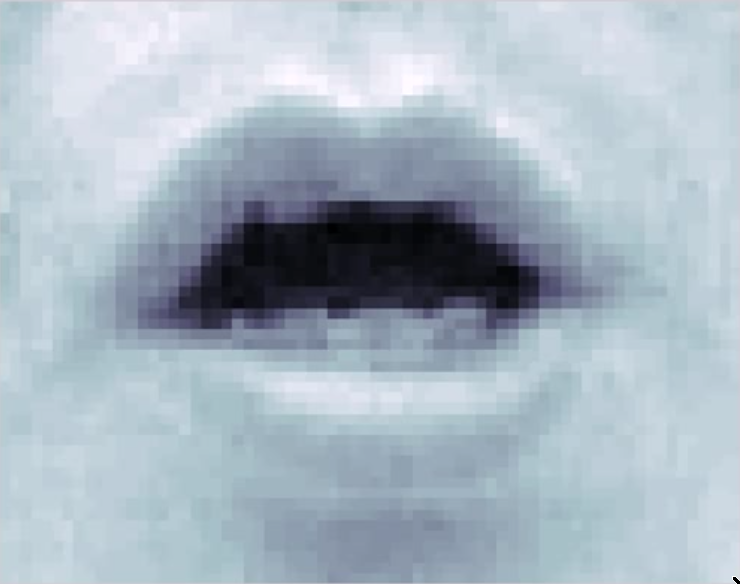}
	\includegraphics[width=.35\textwidth]{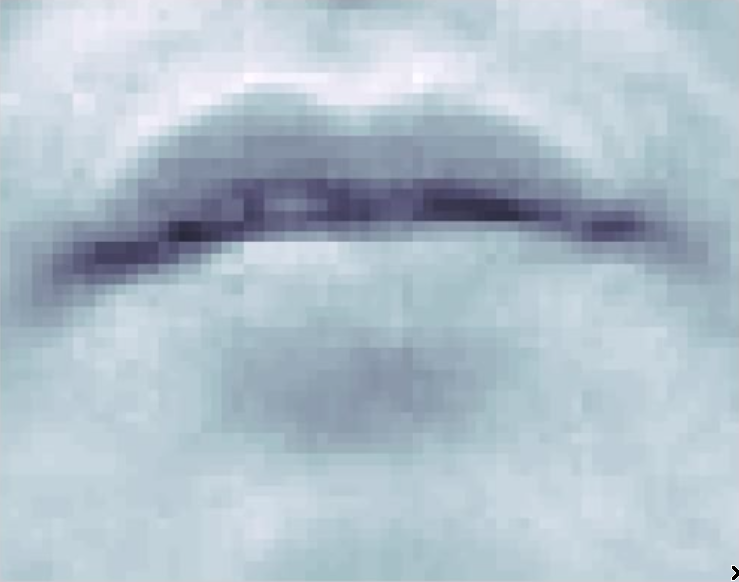}
	\caption{Two example frames from lip video dataset}
\end{figure}

\subsubsection{Methodology and Results}\label{bhag:methods}

For this data set, we test the two kernels: the Grassmann kernel as used above as well as a standard
radial-basis function (RBF) kernel, also referred to as the Gaussian kernel on Euclidean space, to establish a performance baseline using a support
vector machine to perform classification. 
Both the kernels  and the corresponding hyperparameters are defined in Table
\ref{table:kerntablenaoki}. Each frame of the video clips is flattened into a 3850 x 1 row vector,
yielding a 3850 x $\tau$ signal for each clip, where $\tau$ denotes the number of frames in the
clip. For the RBF kernel, we further flatten the signal into a (3850 x $\tau$) x 1 row vector. In
accordance to the experiments in \cite{lieu_signal_2011}, the training set consists of 50\% data and
the test data is the remaining 50\%. Specifically, we randomly choose 5 videos from each of the five
digit classes to use as training data, totalling 25 videos. We run each test 25 times with
randomly-chosen splits each time and average the results below. We compare the performance of our
methods to the PCA-EMD algorithm used by Lieu and Saito \cite{lieu_signal_2011}, where the data is transformed into a
lower-dimensional representation using principal component analysis, then the earth-mover's distance
is calculated and used for nearest-neighbor classification.

\begin{table}[ht]
	\begin{center}
		\begin{tabular}{|c|c|c|}
			\hline
			Kernel Type         & Expression                                & Parameters\\
			\hline
			Grassmann kernel    & $\exp(-\gamma (r - \|U_1^TU_2\|_F^2))$    & $\gamma = 0.2$ \\ \hline
			Euclidean kernel    & $\exp(-\gamma \|x - x'\|^2)$              & $\gamma = 2.1 \cdot 10^{-7}$ \\ \hline
		\end{tabular}
	\end{center}
	\caption{The definition of various kernels used for lip video dataset.}
	\label{table:kerntablenaoki}
\end{table}

\begin{figure}[ht]
	\centering
	\begin{tabular}{|l|rrr|}
		\hline
		Approach        & Average Accuracy (\%) & Train Time (s)    & Test Time (s) \\ \hline \hline
		Grassmann SVM   & 95.84$\pm$3.29        & 3.06$\pm$0.39     & 3.15$\pm$0.48 \\
		Euclidean SVM   & 75.20$\pm$8.31        & 2.43$\pm$0.28     & 2.51$\pm$0.47 \\
		PCA-EMD \cite{lieu_signal_2011} & 94.70 & & \\
		PCA-HD \cite{lieu_signal_2011}  & 90.60 & & \\ \hline
	\end{tabular} \vspace{.1in} 
	\caption{Classification accuracy of three kernels on lip video dataset}
	\label{fig:lip_table}
\end{figure}

The Grassmann kernel performs the best at 95.84\% accuracy, outperforming the two methods (PCA-EMD
and PCA-HD) originally used on the dataset, as well as the basic Euclidean kernel.

As above, we perform the experiment using differently-sized subsets of training data for the
classifiers to evaluate the robustness of each approach to small training sets. Each approach is
evaluated on \{10, 20, 40, 60, 80\}\% of the training data for 25 trials each, and each train/test
split retains an equal ratio of each class. The results are shown in Figure \ref{fig:lip_sizes}.
Similar to the DopNet dataset, performance does not vary significantly within 40-80\% training size
using the Grassmann SVM method, suggesting that the Grassmann kernel has strong predictive power with
just 20 samples in the training set.

\begin{figure}
	\centering
	\includegraphics[width=.45\linewidth]{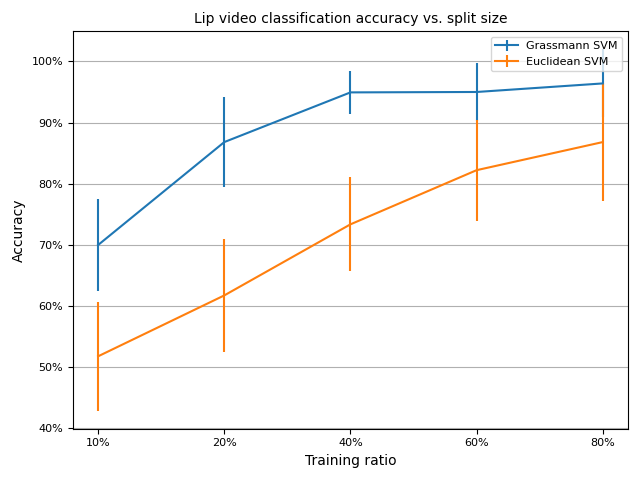}
	\caption{Classification accuracy vs. training/test data set split ratio on lip video dataset. In this plot the $y$-axis is $40\%$-$100\%$.}
	\label{fig:lip_sizes}
\end{figure}

\bhag{Proofs}\label{bhag:proofs}

The proofs of both Theorems~\ref{theo:empirical} and \ref{theo:theoretical} involve the solution of an interpolation problem. 
Accordingly, our first objective is to study this problem when $\Phi_N$ is an admissible kernel (cf. Definition~\ref{def:admissibledef}).
We assume the set up as in Section~\ref{bhag:main}.

For integer $M\ge 2$, distinct points $y_1,\cdots,y_M \in \XX$, and real numbers $f_1,\cdots,f_M$, we consider the interpolation problem
\be\label{eq:interp_problem}
\sum_{k=1}^M a_k\Phi_N(y_j,y_k)=f_j.
\ee
Our first goal is to investigate when the collocation/normal matrix  $[\Phi_N(y_j,y_k)]_{j,k=1}^M$  is invertible and to estimate the norm of this matrix.

\begin{theorem}\label{theo:matrixinv}
	Let $M\ge 2$ be an integer, $\C=\{y_1,\cdots,y_M\}\subset\XX$, $\eta=\eta(\C)$ be defined as in \eqref{eq:minsep}, $\tilde{\eta}=\min(1,\eta/3)$. 
	We assume that $\Phi_N$ is a $(q,S)$-localized kernel that   satisfies \eqref{eq:ondiagonal}, and further that \eqref{eq:ballmeasurecond} and \eqref{eq:ballmeasurelow} are satisfied.
	There exists a positive constant $C$  such that if  $N \ge C\tilde{\eta}^{-1}$, then  the system of equations \eqref{eq:interp_problem} has a solution $\{a_k^*\}$ satisfying
	\be\label{eq:interp_est}
	\max_{1\le k\le M}|a_k^*| \ls N^{-q}\max_{1\le j\le M}|f_j|,
	\ee
	and
	\be\label{eq:interp_errest}
	\max_{1\le k\le M} \left|a_k^*\Phi_N(y_k,y_k)-f_k\right|\ls (N\eta)^{q-S}\max_{1\le j\le M}|f_j|.
	\ee 
\end{theorem}

It is convenient to formulate certain technical details of the proof in terms of the notion of a regular measure.
If $\nu$ is a (positive or signed) measure on $\XX$,  we denote by $|\nu|$ its total variation measure. If $d\ge 0$, we say that $\nu$ is $d$-regular if
\be\label{eq:upballmeasure}
\tn\nu\tn_d=\sup_{\x\in\XX\atop r>0}\frac{|\nu|(\BB(x,r)}{(r+d)^q} <\infty.
\ee
For example, $\mu^*$ itself is a $0$-regular measure, with $\tn\mu^*\tn_0$ being the constant involved in \eqref{eq:ballmeasurecond}.
Another important example is given in the following lemma.

\begin{lemma}\label{lemma:reg_examp_lemma}
	Let $\C=\{y_1,\cdots,y_M\}$ be distinct points in $\XX$, $\eta=\eta(\C)$ be as in \eqref{eq:minsep}, $\tilde{\eta}=\min(1,\eta/3)$, and $\nu$ be the counting  measure; i.e., the measure that associates the mass $1$ with each $y_j$. 
	Then 
	\be\label{eq:counting_reg}
	\tn\nu\tn_{\tilde{\eta}}\ls \frac{1}{\tilde{\eta}^q}.
	\ee 
\end{lemma}
\begin{proof}\ 
	Let $x\in\XX$, $r>0$, and $J$ be the number of points in $\C\cap \BB(x,r)$. 
	By re-arranging the indices, we may assume without loss of generality that $\C\cap\BB(x,r)=\{y_1,\cdots,y_J\}$.
	Then the balls $\BB(y_j,\tilde{\eta})$ are pairwise disjoint and the union of these balls is a subset of $\BB(x,r+\tilde{\eta})$.
	Therefore, using the conditions \eqref{eq:ballmeasurelow} and \eqref{eq:ballmeasurecond} on measures of balls, we deduce that
	$$
	J\tilde{\eta}^q \ls \sum_{j=1}^J \mu^*\left(\BB(y_j,\tilde{\eta})\right)\le \mu^*\left(\cup_{j=1}^J \BB(y_j,\tilde{\eta})\right)\le \mu^*(\BB(x, r+\tilde{\eta}))\ls (r+\tilde{\eta})^q.
	$$
	This implies \eqref{eq:counting_reg}.
\end{proof}
\begin{lemma}\label{lemma:fundalemma}
	Let $d\ge 0$ and $\nu$ be a $d$-regular measure.
	Let $\Phi_N$ be a $(q,S)$-localized kernel.
	If $N\ge 1$ and $r\ge 1/N$, then 
	\be\label{eq:extintest}
	\sup_{x\in\XX}\int_{\Delta(x,r)}|\Phi_N(x,y)|d|\nu|(y) \ls \tn\nu\tn_d (Nr)^{q-S}(1+d/r)^q.
	\ee
	Hence,
	\be\label{eq:lebesgue_est}
	\sup_{x\in\XX}\int_{\XX}|\Phi_N(x,y)|d|\nu|(y)\ls \tn\nu\tn_d(1+Nd)^q,
	\ee
	and
	\be\label{eq:l2normest}
	\sup_{x\in\XX}\int_{\XX}|\Phi_N(x,y)|^2d|\nu|(y)\ls \tn\nu\tn_d N^q(1+Nd)^q
	\ee
	In particular,  taking $d\nu=gd\mu^*$, where $g\in C(\XX)$, we have $d=0$, $\tn\nu\tn_0\ls \|g\|_\infty$, and obtain for $N >0$
	\be\label{eq:mulebesgue}
	\sup_{x\in\XX}\int_{\XX}|\Phi_N(x,y)||g(y)|d\mu^*(y)\ls \|g\|_\infty.
	\ee
\end{lemma}
\begin{proof}\ 
	Without loss of generality, we may replace $|\nu|/\tn\nu\tn_d$ by $\nu$, and thereby, assume that $\nu$ is a positive measure with $\tn\nu\tn_d=1$.
	In this proof only, let $x\in\XX$, $A_k=\{y\in \XX: 2^kr\le \rho(x,y)<2^{k+1}r\}$. Clearly, \eqref{eq:extintest} implies that $\nu(A_k)\ls 2^{kq}r^q(1+d/r)^q$. Therefore, using \eqref{eq:genlocalest}, we deduce that
	\begin{eqnarray*}
		\int_{\Delta(x,r)}|\Phi_N(x,y)|d|\nu|(y)&=&\sum_{k=0}^\infty \int_{A_k}|\Phi_N(x,y)|d|\nu|(y)\\
		&\ls& N^q\sum_{k=0}^\infty \int_{A_k}\frac{d\nu(y)}{(N\rho(x,y))^S}\ls N^{q-S}r^{-S}\sum_{k=0}^\infty2^{-kS} \int_{A_k}d\nu(y)\\
		&\ls&(Nr)^{q-S}(1+d/r)^q\sum_{k=0}^\infty2^{-k(S-q)}.
	\end{eqnarray*}
	This proves \eqref{eq:extintest}.
	To prove \eqref{eq:lebesgue_est}, we observe that
	$$
	\int_{\BB(x,1/N)}|\Phi_N(x,y)|d\nu(y)\ls N^q\nu(\BB(x,1/N))\ls N^q(1/N+d)^q= (1+Nd)^q.
	$$
	The same estimate is obtained for the integral over $\Delta(x,1/N)$ by using \eqref{eq:extintest} with $r=1/N$. 
	We arrive at \eqref{eq:lebesgue_est} by adding these estimates.
	Since \eqref{eq:genlocalest} shows that $|\Phi_N(x,y)|\ls N^q$ for all $x, y\in\XX$, we get from \eqref{eq:lebesgue_est} that for any $x\in\XX$,
	$$
	\int_{\XX}|\Phi_N(x,y)|^2d|\nu|(y)\ls N^q\int_{\XX}|\Phi_N(x,y)|d|\nu|(y)\ls N^q(1+Nd)^q.
	$$
\end{proof}\\

\noindent
\textsc{Proof of Theorem~\ref{theo:matrixinv}.}\\

In this proof, we simplify our notation, and write $\eta$ in place of $\tilde{\eta}$. 
We consider first the measure $\nu$  as in Lemma~\ref{lemma:reg_examp_lemma}, so that  with $d=\eta$, $\tn\nu\tn_d \ls \eta^{-q}$. 
If $N\eta \ge 1$, then we may use \eqref{eq:extintest} with $r=\eta$ to obtain
\be\label{eq:pf1eqn1}
\max_{1\le k\le M}\sum_{j\not= k}|\Phi_N(y_j,y_k)| \le \int_{\Delta(y_j,\eta)}|\Phi_N(y_j,y)|d\nu(y) \ls N^q(N\eta)^{-S}(1+N\eta)^q\ls N^q (N\eta)^{q-S}.
\ee
In view of our assumption \eqref{eq:ondiagonal}, $\Phi_N(y_j,y_j)\gs N^q$. 
Further, we recall that $S>q$.
So,  if $N\eta\gs 1$, the matrix $\mathbf{L}=[\Phi_N(y_k,y_k)]_{j,k=1}^M$ satisfies
$$
\max_{1\le k\le M}\sum_{j\not k}|\Phi_N(y_j,y_k)| \le (1/2)\Phi_N(y_k,y_k), \qquad k=1,\cdots, M.
$$
The facts that $\mathbf{L}$ is invertible, and \eqref{eq:interp_est} holds,
now follow from well known facts in linear algebra, e.g., \cite[Proposition~6.1]{eignet}.
The estimate \eqref{eq:interp_errest} is clear from \eqref{eq:interp_est} and \eqref{eq:pf1eqn1}.
\qed\\

\noindent\textsc{Proof of Theorem~\ref{theo:empirical}.}\\
The conditions of Theorem~\ref{theo:empirical} ensure those in Theorem~\ref{theo:matrixinv} are satisfied with the kernel $\Phi_N$. 
Hence, a solution of the equation in \eqref{eq:riskmin} exists and therefore, is automatically the minimizer of the  empirical risk. 
It remains to obtain the other bounds in the theorem as indicated.
Once more, we denote $\tilde{\eta}$ by $\eta$. We consider first the case when $\delta(x)\ge \eta/3$. Then using Lemma~\ref{lemma:fundalemma} with $\nu$ being the measure as in Lemma~\ref{lemma:reg_examp_lemma}, and then using \eqref{eq:interp_est} we obtain
\be\label{pf4eqn1}
\left|\sum_{k=1}^M a_k^*\Phi_N(x,y_k)\right| \ls \|F\|_\infty N^{-q}\sum_{k=1}^M |\Phi_N(x,y_k)| \ls \|F\|_\infty N^{-q}\int_{\Delta(x,\eta/3)}|\Phi_N(x,y)|d\nu(y) \ls \|F\|_\infty (N\eta)^{-S}.
\ee
Next, we consider the case when $\delta(x)\le \eta/3$. In this case, there is a unique $y_\ell$ with $\rho(x,y_\ell)=\delta(x)$, and $y_k\in \Delta(x,\eta/3)$ for all $k\not=\ell$. 
Arguing as before, we see that 
\be\label{pf4eqn2}
\left|a_\ell^*\Phi_n(x,y_\ell)-\sum_{k=1}^M a_k^*\Phi_n(x, y_k)\right|\ls \|F\|_\infty (N\eta)^{-S}.
\ee
Therefore, using the Lipschitz conditions \eqref{eq:kernellip}, \eqref{eq:interp_est}, \eqref{eq:interp_errest},    we conclude that
$$
\begin{aligned}
	\left|\sum_{k=1}^M a_k^*\Phi_n(x, y_k)\right.&- \left. F(x)\right|\ls |a_\ell^*\Phi_n(x,y_\ell)-F(x)| +  \|F\|_\infty (N\eta)^{-S}\\
	&\le  |a_\ell^*\Phi_n(x,y_\ell)-a_\ell^*\Phi_n(y_\ell,y_\ell)|+|a_\ell^*\Phi_n(y_\ell,y_\ell)-F(y_\ell)| +|F(y_\ell)-F(x)|+  \|F\|_\infty (N\eta)^{-S}\\
	&\ls (N+\|F\|_{\mbox{Lip}})\delta(x)  +  \|F\|_\infty (N\eta)^{q-S}.
\end{aligned}
$$
\qed

In order to prove Theorem~\ref{theo:theoretical}, we prove first the following lemmas, which will enable us to apply Theorem~\ref{theo:empirical} with $\sigma_N(Ff_0)$ (cf. \eqref{eq:gensigmadef}) in place of $F$ and $\Psi_N$ defined in \eqref{eq:psindef} in place of $\Phi_N$. 

\begin{lemma}\label{lemma:psinloclemma}
	Let $\{\Phi_N\}$ be a $(q,S)$-localized family of kernels, and $\Psi_N$ is defined by \eqref{eq:psindef}. 
	If $d\tau=f_0d\mu^*$ for some $f_0\in C(\XX)$, then   $\{\Psi_N\}$ is a $(q,S)$-localized family of kernels.
\end{lemma}

\begin{proof}\ 
	In this proof, let $y\not=x\in\XX$, and $\delta=\rho(x,y)/3$. 
	Then $\BB(x,\delta)\subseteq \Delta(y,\delta)$ and 
	\be\label{pf2eqn1}
	\begin{aligned}
		|\Psi_N(x,y)|&= \left|\int_\XX \Phi_N(x,z)\Phi_N(y,z)d\tau(z)\right|\\
		& \le   \int_{\Delta(x,\delta)}|\Phi_N(x,z)\Phi_N(y,z)|d\tau(z)+\int_{\BB(x,\delta)}|\Phi_N(x,z)\Phi_N(y,z)|d\tau(z)\\
		&\le \int_{\Delta(x,\delta)}|\Phi_N(x,z)\Phi_N(y,z)|d\tau(z) +\int_{\Delta(y,\delta)}|\Phi_N(x,z)\Phi_N(y,z)|d\tau(z).
	\end{aligned}
	\ee
	In view of \eqref{eq:genlocalest},
	\be\label{pf2eqn2}
	|\Phi_N(x,z)| \ls \frac{N^q}{\max(1,(N\rho(x,z))^S)}\ls \frac{N^q}{\max(1,(N\delta)^S)}, \qquad z\in\Delta(x,\delta).
	\ee
	The estimates \eqref{eq:mulebesgue}, \eqref{pf2eqn2} and \eqref{pf2eqn1} lead to the fact that $\Psi_n$ is $(q,S)$-localized.
	
	Next, using \eqref{eq:genlocalest} with $x=y$, we obtain
	$$
	\Psi_N(x,x)=\int_\XX |\Phi_N(x,z)|^2d\tau(z)\ls N^q \int_\XX|\Phi_N(x,z)||f_0(z)|d\mu^*(z).
	$$
	Using \eqref{eq:mulebesgue}, we deduce that the last integral is $\ls \|f_0\|_\infty$.
	Therefore, $\Psi_N(x,x)\ls N^q\|f_0\|_\infty$.
\end{proof}

\begin{lemma}\label{lemma:psinadminlemma}
	Let $\{\Phi_N\}$ be an admissible family of kernels, and $\Psi_N$ is defined by \eqref{eq:psindef}. 
	If $d\tau=f_0d\mu^*$ for some $f_0\in C(\XX)$, and $f_0(\x)\ge \mathfrak{m}(f_0)>0$ for $\x\in\XX$ then   $\{\Psi_N\}$ is an admissible family of kernels.
\end{lemma}

\begin{proof}\ 
	We  have already proved in Lemma~\ref{lemma:psinloclemma} that $\{\Psi_N\}$ is $(q,S)$-localized. 
	In order to prove that \eqref{eq:ondiagonal} holds with $\Psi_N$ replacing $\Phi_N$, we need only to prove that
	\be\label{pf3eqn1}
	\Psi_N(x,x) \gs N^q, \qquad x\in\XX.
	\ee
	Since $\Phi_N$ is admissible, we deduce from  \eqref{eq:kernellip} that that there exists $c>0$ such that if $\rho(x,z)\le c/N$ then
	$$
	|\Phi_N(x,x)-\Phi_N(x,z)|\ls N^{q+1}\rho(x,z)\le (1/2) |\Phi_N(x,x)|; \mbox{   i.e.,   } |\Phi_N(x,z)|\ge (1/2) |\Phi_N(x,x)| \gs N^q.
	$$
	Using \eqref{eq:ondiagonal}, we obtain
	$$
	\Psi_N(x,x)=\int_\XX |\Phi_N(x,z)|^2f_0(z)d\mu^*(z)\ge \mathfrak{m}(f_0)\int_{\BB(x,c/N)}|\Phi_N(x,z)|^2d\mu^*(z)\gs N^{2q}\mathfrak{m}(f_0)\mu^*\left(\BB(x,c/N)\right).
	$$
	In view of \eqref{eq:ballmeasurelow}, this leads to \eqref{pf3eqn1}.
	
	Next, if $x,x', y\in\XX$, then using \eqref{eq:kernellip} and \eqref{eq:mulebesgue}, we obtain
	$$
	\left|\Psi_N(x,y)-\Psi_n(x',y)\right| \le \int_\XX \left|\Phi_N(x,z)-\Phi_N(x',z)\right||\Phi_N(z,y)||f_0(z)|d\mu^*(z)
	\ls \|f_0\|_\infty N^{q+1}\rho(x,x').
	$$
	This proves \eqref{eq:kernellip} with $\Psi_N$ in place of $\Phi_N$.
\end{proof}\\

\noindent\textsc{Proof of Theorem~\ref{theo:theoretical}.}\\

Since $\Psi_N$ is an admissible kernel, we may apply 
Theorem~\ref{theo:matrixinv} with $\Psi_N$ replacing $\Phi_N$ and $\sigma_N(Ff_0)(y_j)$ in place of $f_j$ (cf. \eqref{eq:gensigmadef}) to obtain $\{a_k^*\}$ such that
\be\label{pf5eqn1}
\sum_{k=1}^M a_k^*\Psi_N(y_k,y_j)=\sigma_N(Ff_0)(y_j)=\int_\XX \Phi_N(y_j,y)F(y)f_0(y)d\mu^*(y), \qquad j=1,\cdots, M,
\ee
and moreover, (cf. \eqref{eq:interp_est} and \eqref{eq:interp_errest})
\be\label{pf5eqn2}
\max_{1\le k\le M}|a_k^*| \ls N^{-q}\max_{1\le j\le M}|\sigma_N(Ff_0)(y_j)|,
\ee
and
\be\label{pf5eqn3}
\max_{1\le k\le M} \left|a_k^*\Psi_N(y_k,y_k)-\sigma_N(Ff_0)(y_k)\right|\ls (N\eta)^{q-S}\max_{1\le j\le M}|\sigma_N(Ff_0)(y_j)|.
\ee 
We note that the minimizer $P_T(\tau;\mathcal{V}(C);F)$ of the theoretical least square loss is given by
\be\label{pf5eqn4}
P_T(\tau;\mathcal{V}(C);F)(x)=\sum_{k=1}^M a_k^*\Phi_N(x,y_k).
\ee
Further, in view of \eqref{eq:mulebesgue},
\be\label{pf5eqn5}
\|\sigma_N(Ff_0)\|_\infty \ls \|Ff_0\|_\infty.
\ee
Let $x\in\XX$ and $\delta(x)>\eta/3$. 
Then using the measure $\nu$ as in Lemma~\ref{lemma:reg_examp_lemma}, we obtain from \eqref{pf5eqn2}, \eqref{eq:extintest}, and \eqref{pf5eqn5} that
\be\label{pf5eqn6}
\begin{aligned}
	|P_T(\tau;\mathcal{V}(C);F)(x)|&\le \sum_{k: y_k\in \Delta(x,\eta/3)} |a_k^*||\Phi_N(x,y_k)|\ls N^{-q}\|\sigma_N(Ff_0)\|_\infty\int_{y\in \Delta(x,\eta/3)}|\Phi_N(x,y)|d\nu(y)\\
	&\ls N^{-q}\|\sigma_N(Ff_0)\|_\infty\eta^{-q}(N\eta)^{q-S} \ls \|Ff_0\|_\infty (N\eta)^{-S}.
\end{aligned}
\ee
This proves \eqref{eq:theo_away}.

Next, let $\delta(x)\le \eta/3$. 
Then there exists a unique $\ell$ such that $\delta(x)=\rho(x,y_\ell)$. So, using the fact that $\Psi_N(y_\ell,y_\ell)\sim \Phi_N(y_\ell,y_\ell)\sim N^q$,  \eqref{pf5eqn3}, and \eqref{pf5eqn5}, we obtain
$$
\left|a_\ell^*\Phi_n(y_\ell,y_\ell)-\frac{\Phi_n(y_\ell,y_\ell)}{\Psi_n(y_\ell,y_\ell)}\sigma_N(Ff_0)(y_\ell)\right| \ls (N\eta)^{q-S}\|\sigma_N(Ff_0)\|_\infty\ls (N\eta)^{q-S}\|Ff_0\|_\infty,
$$
and hence,
\be\label{pf5eqn7}
\left|a_\ell^*\Phi_n(y_\ell,y_\ell)-\frac{\Phi_n(y_\ell,y_\ell)}{\Psi_n(y_\ell,y_\ell)}F(y_\ell)f_0(y_\ell)\right| \ls \|Ff_0-\sigma_N(Ff_0)\|_\infty + (N\eta)^{q-S}\|Ff_0\|_\infty.
\ee
Since
$$
|F(y_\ell)f_0(y_\ell)-F(x)f_0(x)|\le \|Ff_0\|_{\mbox{Lip}}\delta(x),
$$
we deduce from \eqref{pf5eqn7}, \eqref{eq:kernellip}, \eqref{pf5eqn2},   and \eqref{pf5eqn5} that
\be\label{pf5eqn8}
\left|a_\ell^*\Phi_n(x,y_\ell)-\frac{\Phi_n(y_\ell,y_\ell)}{\Psi_n(y_\ell,y_\ell)}F(x)f_0(x)\right| \ls (N+\|Ff_0\|_{\mbox{Lip}})\delta(x)+ \|Ff_0-\sigma_N(Ff_0)\|_\infty + (N\eta)^{q-S}\|Ff_0\|_\infty.
\ee
Arguing as in \eqref{pf5eqn6}, we get
$$
\sum_{k\not=\ell}|a_k^*\Phi_n(x, y_k)| \ls \|Ff_0\|_\infty (N\eta)^{-S}.
$$
Hence, \eqref{pf5eqn8} leads to \eqref{eq:theoretical_err_est}. \qed

\bhag{Conclusions}\label{bhag:conclusions}
In \cite{mhaskar2019deep}, HNM had developed a very simple method to approximate functions on unknown manifolds without making any effort to learn the manifold itself (e.g., by estimating an atlas or eigen-decomposition of the Laplace-Beltrami operator). 
This method involves a simple matrix vector multiplication using a specially constructed localized kernel. 
However, the approach requires that the training data be dense on the manifold. 
In this paper, we examine the accuracy of the approximation if the training data is sparse instead, and we use either empirical risk minimization or the theoretical  square loss minimization. 
We study this question in a very general setting of a locally compact metric measure space, thereby initializing a theme for further research where the unknown manifold is known to be a sub-manifold of a known manifold rather than just a high dimensional Euclidean space.
In practice, the problem arises, for example, in analysis of time series, for which the domain knowledge indicates that the ambient space is a Grassmann manifold.
We present a detailed experimental study where we use different variations  of the localized kernel (\eqref{fastkerndef}) to the classification of hand gestures using micro-doppler radar data - a problem of interest in its own right. 
Our results show that the SVMs trained with our proposed localized kernel and PCA components of zero-padded vectorized spectrograms outperform existing methods for micro-Doppler gesture recognition, including some CNNs by a $3\%-6\%$ margin of accuracy, but with a much shorter training time. To demonstrate the fact that our theory is general purpose, we use similar techniques for the classification of spoken digits from a video data set, and demonstrate how an embedding of the data set onto an unknown submanifold of a Grassmann manifold yields superior results.

\section*{Acknowledgments}
We thank Professor Dr. Naoki Saito at University of California, Davis for providing us with the data set used in \cite{lieu_signal_2011}.
We thank the referees for their valuable suggestions for the improvement of this paper.

\bibliographystyle{abbrv}
\bibliography{doppler_bib}

\end{document}